\documentclass{article}

\PassOptionsToPackage{numbers, compress}{natbib}

\usepackage[main, final]{neurips_2026}

\usepackage[utf8]{inputenc} 
\usepackage[T1]{fontenc}    
\usepackage{hyperref}       
\usepackage{url}            
\usepackage{booktabs}       
\usepackage{amsfonts}       
\usepackage{nicefrac}       
\usepackage{microtype}      
\usepackage{xcolor}         

\usepackage{enumitem}
\usepackage{amsmath,amssymb,amsthm}
\usepackage[algo2e,ruled,vlined]{algorithm2e}

\usepackage{natbib}

\newcommand{\op}[1]{\operatorname{#1}}
\newcommand{\C}[1]{{\mathcal{#1}}} 
\newcommand{\B}[1]{{\mathbb{#1}}} 
\newcommand{\BF}[1]{{\mathbf{#1}}} 
\newcommand{\F}[1]{{\mathfrak{#1}}}

\newcommand{\bra}{\langle}
\newcommand{\ket}{\rangle}

\newcommand{\x}{\BF{x}}
\newcommand{\y}{\BF{y}}
\newcommand{\z}{\BF{z}}
\newcommand{\vv}{\BF{v}}
\newcommand{\uu}{\BF{u}}

\renewcommand{\cite}[1]{\citep{#1}}

\DeclareMathOperator*{\argmax}{argmax}
\usepackage{tikz}
\usetikzlibrary{positioning,calc}
\usepackage{wrapfig}

\newtheorem{definition}{Definition}
\newtheorem{remark}{Remark}
\newtheorem{lemma}{Lemma}
\newtheorem{proposition}{Proposition}
\newtheorem{theorem}{Theorem}
\newtheorem{corollary}{Corollary}
\newtheorem{example}{Example}

\title{$\gamma$-weakly $\theta$-up-concavity: A Unified Framework for Non-Convex Optimization Beyond DR-Submodular and OSS Functions}

\author{%
  Mohammad Pedramfar \\
  Mila -- Quebec AI Institute / McGill University \\
  \texttt{mohammad.pedramfar@mila.quebec} \\
  \And
  Vaneet Aggarwal \\
  Purdue University \\
  \texttt{vaneet@purdue.edu} \\
}

\begin{document}

\maketitle

\begin{abstract}
Many objectives in machine learning and combinatorial optimization are non-concave. We introduce $\gamma$-weak $\theta$-up-concavity, a first-order condition for nonnegative monotone functions in which a scaling function $\theta$ lets the curvature change with the size of the solution. The resulting class strictly contains monotone DR-submodular functions and One-Sided Smooth (OSS) functions. It also contains functions whose marginal gains first increase and then decrease, and functions that vanish faster than any power of the norm at the origin. We construct a surrogate gradient from gradients along the segment from the origin to the current point and prove an upper-linearization inequality with a point-dependent approximation coefficient. Restricting actions to the closed convex hull of the coordinate-wise maximal feasible points makes this coefficient uniformly positive, with a value determined by the curvature parameters and the geometry of the feasible set. Standard reductions to linear optimization then give offline approximation algorithms and online algorithms with regret guarantees. The coefficient matches the optimal one for DR-submodular maximization and improves the best known coefficients for OSS maximization over matroid constraints whenever the OSS parameter is positive.
\end{abstract}

 \section{Introduction}
\label{sec:intro}

Maximizing nonlinear objectives over convex sets arises in influence maximization, recommender systems, experimental design, feature selection, and diversity maximization~\cite{abbassi2013diversity,bian2019optimal,mirzasoleiman2018streaming,umrawal2023community}.
We consider problems of the form $\max_{\x \in \C{K}} f(\x)$, where $\C{K}$ is a convex set and $f$ belongs to a class $\C{F}$ of nonnegative functions.
Such objectives are often non-concave and the problems are often NP-hard, so we look for a polynomial-time algorithm that finds $\y \in \C{K}$ with $f(\y) \geq \alpha \max_{\x \in \C{K}} f(\x)$, where the \emph{approximation coefficient} $\alpha \in (0, 1]$ is as large as possible.

Several relaxations of concavity keep these problems tractable in continuous domains, including continuous DR-submodularity~\cite{bian2019optimal,hassani17_gradien_method_submod_maxim}, up-concavity~\cite{mitra2021submodular,pedramfar2024linear,pedramfar2025uniform}, and one-sided smoothness (OSS)~\cite{ghadiri2025beyond,zhang2024online,zhang2023parallelized}.
They have led to approximation algorithms for offline and online optimization under downward-closed and matroid constraints.
\citet{pedramfar2024linear} introduced \emph{upper-linearizable} function classes and meta-algorithms that transfer guarantees for linear objectives to these classes by bounding approximation regret with the regret of suitable linear surrogates.
Upper-linearizable classes include DR-submodular functions and phase retrieval objectives; finding broad, verifiable classes with good upper-linearization constants is open.

We introduce \emph{$\gamma$-weak $\theta$-up-concavity}, a first-order condition for monotone objectives that strictly generalizes $\gamma$-weakly DR-submodular functions, OSS functions, and other weakly concave families.
It bounds increments of $f$ by gradient inner products weighted by a monotone scaling function $\theta$, which describes how the curvature changes with the scale of the solution.
In many allocation problems, marginal gains first increase and then decrease because of saturation or interference; we call this \emph{accumulating-then-diminishing returns}.
Examples include recommendation and influence systems, where engagement may accelerate as exposure accumulates, and allocation problems in which marginal gains are very small at first.
Some objectives are also \emph{flat at the origin}: they vanish faster than any power of $\|\x\|_1$ as $\x \to \BF{0}$.
OSS models allow some of these behaviors but do not describe scale-dependent curvature in general.
Appendix~\ref{apx:beyondoss} gives functions with both behaviors that are $\gamma$-weakly $\theta$-up-concave but neither DR-submodular nor $\sigma$-OSS for any parameters.

We show that every $\gamma$-weakly $\theta$-up-concave function is upper-linearizable, with coefficients that depend only on the curvature parameters and the geometry of the feasible set.
For $\theta(\x) = \|\x\|_1^\kappa$ and a matroid constraint, maximizing over the independence polytope has coefficient $1-\exp\!\big(-\frac{\gamma}{2^\kappa (\kappa+1)}\big)$; for $\kappa = 0$ and $\gamma = 1$ this is the optimal coefficient $1 - 1/e$ for DR-submodular maximization.
Every nonnegative monotone $\sigma$-OSS function satisfies our condition with $\gamma = 1$ and $\theta(\x) = \|\x\|_1^{2\sigma}$ (Theorem~\ref{thm:oss-equivalence}), which gives the coefficient $1-\exp\!\big(-\frac{1}{2^{2\sigma}(2\sigma+1)}\big)$; for $\sigma > 0$ this improves the best known coefficient for general monotone OSS objectives~\cite{ghadiri2025beyond,zhang2024online} (Section~\ref{sec:matroid}).
For general downward-closed sets, the coefficient depends on a geometric ratio of the feasible set (Definition~\ref{def:r}) and can be smaller than that of~\cite{ghadiri2025beyond}.

Our main contributions are as follows.
\begin{itemize}[leftmargin=0pt,labelindent=0pt,labelwidth=!,itemindent=6pt,align=left,itemsep=2pt,topsep=2pt,parsep=0pt]
  \item \textbf{A unified curvature model.}
  We define the class (Section~\ref{sec:setup}) and relate it to existing ones. The inclusions relating OSS functions to this class (Theorem~\ref{thm:oss-equivalence}) and the separation examples, including multivariate ones, are new.

  \item \textbf{Upper-linearization guarantees.}
  Every $\gamma$-weakly $\theta$-up-concave function is upper-linearizable. When $\theta$ is small near the origin, a coefficient that is uniform over $\C{K}$ can be zero, so we prove an inequality whose coefficient depends on the point $\x$ and run the base algorithm on $\C{K}^*$, the closed convex hull of the maximal points of $\C{K}$ (Definition~\ref{def:r}), where this coefficient is bounded away from zero.
  We also give the coefficient for any closed convex set between $\C{K}^*$ and $\C{K}$ (Proposition~\ref{prop:outer-approx}).

  \item \textbf{Improved guarantees under matroid constraints.}
  Using the geometry of matroid polytopes, we obtain the coefficients stated above; for $\sigma$-OSS functions, they improve prior coefficients when $\sigma > 0$ (Section~\ref{sec:matroid}).

  \item \textbf{Offline and online optimization via reductions.}
  Through the reductions of~\cite{pedramfar24_unified_framew_analy_meta_onlin_convex_optim,pedramfar2024linear}, upper-linearizability gives offline and online algorithms; the new ingredients are the surrogate gradient of Theorem~\ref{thm:linearizable} and a query algorithm that estimates it without bias from one stochastic gradient.
  This gives (i) a projection-free offline first-order algorithm with sample complexity $\C{O}(\epsilon^{-2})$; (ii) projection-free algorithms with adaptive regret $\C{O}(T^{1/2})$ under full-information and $\C{O}(T^{2/3})$ under semi-bandit stochastic first-order feedback; and (iii) dynamic regret $\C{O}( \sqrt{T(1 + P_T(\uu))} )$ under full-information stochastic first-order feedback for comparator sequences $\uu$ in $\C{K}^*$, where $P_T$ is the path length.
\end{itemize}

\section{Background and Notation}\label{sec:background}

\begin{wrapfigure}{r}{0.42\columnwidth}
\vspace{-\baselineskip}
\begin{minipage}{0.42\columnwidth}
\small
\SetAlgorithmName{Meta-algorithm}{}{}
\begin{algorithm2e}[H]
\caption{\small Online Maximization By Quadratization - $\mathtt{OMBQ}(\mathcal{A}, \mathcal{G}, h)$}
\label{alg:ombq}
\SetKwInOut{Input}{Input}\DontPrintSemicolon
\Input{Base algorithm $\mathcal{A}$, Query algorithm $\mathcal{G}$, Mapping $h: \mathcal{K} \to \mathcal{K}$.}
\For{$t = 1, 2, \dots, T$}{
    Let $\mathbf{x}_t$ be the action chosen by $\mathcal{A}$.\;
    \textbf{Play:} $\mathbf{y}_t = h(\mathbf{x}_t)$. \;
    \textbf{Query:} Run query algorithm $\mathcal{G}$ at $\mathbf{x}_t$ to obtain gradient estimate $\mathbf{g}_t$. \;
    \textbf{Update:} Pass reward vector $\mathbf{g}_t$ to $\C{A}$ to update its state. \;
}
\end{algorithm2e}
\SetAlgorithmName{Algorithm}{}{}
\end{minipage}
\vspace{-\baselineskip}
\end{wrapfigure}

We model online optimization as a repeated game over $T$ rounds between an agent and an adversary on a convex domain $\C{K} \subseteq \B{R}^d$; both players know $T$ and $\C{K}$.
An \emph{oblivious} adversary fixes the functions $(f_t)_{t=1}^T$ and their query oracles before the game starts, while a \emph{fully adaptive} adversary may choose $f_t$ and its query oracle after observing $\x_1, \dots, \x_t$.
In round $t$, the agent selects an action $\x_t \in \C{K}$, the adversary provides a query oracle for the reward function $f_t \in \C{F}$, and the agent queries the oracle at points $\y_{t,i}$, $1 \leq i \leq k_t$, for some $k_t \geq 0$.
A query oracle $\C{Q}_f$ for $f$ is a possibly random map that receives a point $\y \in \C{K}$ and returns information about $f$ near $\y$.
The gradient oracle returns $\C{Q}_f(\y) = \nabla f(\y)$.
More generally, $\C{Q}_f$ is a \emph{stochastic first-order} oracle for $f$ if $\B{E}[ \C{Q}_f(\y) ] = \nabla f(\y)$, where the expectation is conditional on the history of the game and on the query point $\y$.
The oracle is bounded by $B$ if $\| \C{Q}_f(\y) \| \leq B$ always holds, where $\| \cdot \|$ is the Euclidean norm.
An algorithm is \emph{semi-bandit} if it queries the oracle only at its action $\x_t$, and \emph{full-information} otherwise.
For a function class $\C{F}$, $\op{Adv}_1(\C{F})$ denotes the set of fully adaptive adversaries that select functions from $\C{F}$ and provide deterministic first-order oracles, and $\op{Adv}^{\op{o}}_1(\C{F}, B)$ the set of oblivious adversaries that provide stochastic first-order oracles bounded by $B$.

For a maximization problem, $\alpha \in (0, 1]$, $1 \leq a \leq b \leq T$, and a compact set $\C{U} \subseteq \C{K}^T$, the $\alpha$-regret of an algorithm $\C{A}$ against a set of adversaries $\op{Adv}$ is
\begin{align*}
  \C{R}_{\alpha, \op{Adv}}^{\C{A}}(\C{U})[a, b] := 
  \sup_{\C{B} \in \op{Adv}} \B{E} \left[ \alpha \max_{\uu = (\uu_1, \cdots, \uu_T) \in \C{U}} \sum_{t = a}^b f_t(\uu_t) - \sum_{t = a}^b f_t(\x_t) \right],
\end{align*}
where the expectation is over the randomness of the algorithm and of the query oracles.
We drop $\alpha$ when it equals $1$, and assume that the functions are nonnegative when $\alpha < 1$.
With $a=1$, $b=T$, and $\C{U} = \op{diag}(\C{K}^T) := \{(\x, \cdots, \x) \mid \x \in \C{K}\}$, this is the \emph{static regret}.
With $a = 1$, $b = T$, and $\C{U}$ a single comparator sequence $\uu$, it is the \emph{dynamic regret}~\cite{zinkevich03_onlin,zhang18_adapt_onlin_learn_dynam_envir}, which is usually bounded in terms of the path length $P_T(\uu) := \sum_{i=1}^{T-1} \| \uu_i - \uu_{i+1} \|$.
The \emph{adaptive regret} is $\C{AR}_{\alpha, \op{Adv}}^{\C{A}} := \max_{1 \leq a \leq b \leq T} \C{R}_{\alpha, \op{Adv}}^{\C{A}}(\op{diag}(\C{K}^T))[a, b]$~\cite{hazan09_effic}.
We omit $a$, $b$, and $\C{U}$ when they are clear from the context or do not matter.

\label{sec:prelim-linearizable}

We use the upper-linearizable classes of~\cite{pedramfar2024linear}.
Let $\C{K} \subseteq \B{R}^d$ be a convex set.
A class $\C{F}$ of nonnegative functions over $\C{K}$ is \emph{upper-linearizable} if there are maps $\F{g} : \C{F} \times \C{K} \to \B{R}^d$ and $h : \C{K} \to \C{K}$ and constants $0 < \alpha \leq 1$ and $\beta > 0$ such that\footnote{Upper-linearizability generalizes concavity and is used for maximization; lower-linearizability is the analogous notion for minimization.}
\begin{align}\label{eq:linearizable}
\alpha f(\y) - f(h(\x))
\leq \beta  \bra \F{g}(f, \x), \y - \x \ket
\end{align}
for all $f \in \C{F}$ and $\x, \y \in \C{K}$.
Here $\F{g}(f,\x)$ acts as a surrogate gradient, $h$ maps the action of the base algorithm to the action played, and $\alpha$ is the resulting approximation coefficient.
Given a stochastic first-order oracle $\C{Q}_f$ for each $f \in \C{F}$, an algorithm $\C{G}$ is a \emph{query algorithm} for $\F{g}$ if $\C{G}(\C{Q}_f, \x)$ is an unbiased estimate of $\F{g}(f, \x)$; it may query $\C{Q}_f$ several times and at several points.
Upper-linearizable classes can be optimized with the meta-algorithm $\mathtt{OMBQ}$ (Meta-algorithm~\ref{alg:ombq}), which has the following guarantee.

\begin{theorem}[Theorem 1 in~\citet{pedramfar2024linear}]
\label{thm:regret_transfer}
Let $\C{K} \subseteq [0,1]^d$ and let $\mathcal{A}(\C{K})$ be a deterministic algorithm for online optimization with semi-bandit feedback over $\C{K}$.
Also let $\mathcal{F}$ be a function class over $\mathcal{K}$ that is upper-linearizable with $\mathfrak{g}: \mathcal{F} \times \mathcal{K} \to \mathbb{R}^d$ and $h: \mathcal{K} \to \mathcal{K}$.
Let $\mathcal{G}$ be a query algorithm for $\mathfrak{g}$ and let
$\mathcal{A}' = \mathtt{OMBQ}(\mathcal{A}(\C{K}), \mathcal{G}, h)$.
If $\mathcal{G}$ returns an unbiased estimate of $\mathfrak{g}$ and the output of $\mathcal{G}$ is bounded by $B_1$, then
\begin{equation*}
\mathcal{R}^{\mathcal{A}'}_{\alpha, \op{Adv}^{\op{o}}_1(\C{F}, B_1)} 
\le 
\beta \mathcal{R}^{\mathcal{A}(\C{K})}_{1, \op{Adv}_1(\C{L}(\C{K})[B_1])},
\end{equation*}
where $\alpha$ and $\beta$ are the upper-linearization constants and $\C{L}(\C{K})[B_1]$ denotes the class of linear functions $\ell : \C{K} \to \B{R}$ such that $\| \nabla \ell \| \leq B_1$.
\end{theorem}

\citet{pedramfar2024linear} showed that concave functions and several classes of up-concave functions are upper-linearizable. The class introduced in the next section satisfies a version of~\eqref{eq:linearizable} in which $\alpha$ depends on $\x$.

\section{\texorpdfstring{$\gamma$}{gamma}-Weakly \texorpdfstring{$\theta$}{theta}-Up-Concave Functions}
\label{sec:setup}

This section defines the class and relates it to existing ones. We say that a function $\phi : [0,1]^d \to \B{R}_{\geq 0}$ is \emph{flat at the origin} if $\phi(\x) / \|\x\|_1^n \to 0$ as $\x \to \BF{0}$ for every $n > 0$.

A function $f : \C{K} \to \B{R}$ is \emph{monotone} if $f(\y) \geq f(\x)$ for all $\x, \y \in \C{K}$ with $\y \geq \x$. We write $C^1$ and $C^2$ for the continuously differentiable and twice continuously differentiable functions.

Let $\theta : [0, 1]^d \to \B{R}_{\ge 0}$ be a continuous monotone function where $\theta(\x) \neq 0$ for all $\x \in [0,1]^d \setminus \{\BF{0}\}$.
Given $0 < \gamma \leq 1$, we say a $C^1$ monotone function $f :  [0, 1]^d \to \B{R}_{\ge 0}$ is \emph{$\gamma$-weakly $\theta$-up-concave} if
\begin{align}
\tfrac{\gamma\,\theta(\x)}{\theta(\y)}
\langle \nabla f(\y), \y-\x \rangle
&\leq
f(\y) - f(\x)
\leq
\tfrac{\theta(\y)}{\gamma\,\theta(\x)}
\langle \nabla f(\x), \y-\x \rangle,
\label{eq:weak-theta}
\end{align}
for all $\x,\y \in [0,1]^d$ with $\y \geq \x \neq \BF{0}$. If $\theta(\BF{0}) > 0$, the inequalities also hold for $\x = \BF{0}$ by continuity, taking $\x = \epsilon \y$ with $\epsilon \downarrow 0$.
We write $\C{F}_{\gamma, \theta}$ for this class. The ratio $\theta(\y)/\theta(\x)$ lets the curvature bounds change multiplicatively with scale. The class is closed under nonnegative combinations.

\begin{lemma}\label{lem:function-class-convex-cone}
If $f, g \in \C{F}_{\gamma,\theta}$ and $a, b \geq 0$, then $af + bg \in \C{F}_{\gamma,\theta}$.
\end{lemma}
\begin{proof}
The function $h = af + bg$ is $C^1$, nonnegative and monotone, and $\nabla h = a\nabla f + b\nabla g$. Multiplying the inequalities~\eqref{eq:weak-theta} for $f$ by $a$ and those for $g$ by $b$, and adding them, gives~\eqref{eq:weak-theta} for $h$.
\end{proof}

The case $\theta(\x) = \|\x\|_p^\sigma$ with $p \geq 1$ and $\sigma \geq 0$ is important (for $\sigma = 0$, $\theta \equiv 1$, including at the origin). We call such functions \emph{$\gamma$-weakly $(p, \sigma)$-up-concave}; they satisfy
\begin{align}
\tfrac{\gamma\,\|\x\|_p^\sigma}{\|\y\|_p^\sigma}
\langle \nabla f(\y), \y-\x \rangle
&\leq
f(\y) - f(\x)
\leq
\tfrac{\|\y\|_p^\sigma}{\gamma\,\|\x\|_p^\sigma}
\langle \nabla f(\x), \y-\x \rangle.
\label{eq:weak-p-sigma}
\end{align}
for all $\x,\y \in [0,1]^d$ with $\y \geq \x \neq \BF{0}$.
We use $\C{F}_{\gamma, p, \sigma}$ to denote this class of functions.

\paragraph{Normalized Diminishing Returns.}

The next theorem, proved in Appendix~\ref{apx:theta-normalized-DR}, gives a sufficient condition in terms of a normalized gradient.
\begin{theorem}\label{thm:theta-normalized-DR}
Let $\theta:[0,1]^d\to \mathbb{R}_{\ge 0}$ be monotone, continuous, and satisfy
$\theta(\x)>0$ for all $\x \neq 0$.
Suppose $f:[0,1]^d\to\mathbb{R}_{\ge 0}$ is a $C^1$
monotone function where, for all $\x,\y \in [0,1]^d$ with $\y \ge \x \neq 0$,
\begin{align}
\label{eq:normalized-dr}
\frac{\gamma \nabla f(\y)}{\theta(\y)}
\le
\frac{\nabla f(\x)}{\theta(\x)}
\end{align}
coordinate-wise.
Then $f$ is $\gamma$-weakly $\theta$-up-concave.
\end{theorem}

For $\gamma = 1$, condition~\eqref{eq:normalized-dr} says that the normalized gradient $\nabla f / \theta$ is nonincreasing along positive directions; for $\gamma < 1$, it may increase by a factor of at most $1/\gamma$. The gradient $\nabla f$ itself may increase, as long as $\theta$ grows accordingly. Condition~\eqref{eq:normalized-dr} is sufficient but not necessary.

\paragraph{\texorpdfstring{Choosing $\theta$ and $\gamma$.}{Choosing theta and gamma.}}
In a model, $\theta$ describes how the scale of the solution affects marginal gains, and $\gamma$ how far the model may deviate from exact normalized diminishing returns.
Polynomial growth of marginal gains suggests $\theta(\x) = \|\x\|_p^\sigma$; marginal gains that are extremely small at low allocation levels suggest a $\theta$ that is flat at the origin, as in Example~\ref{ex:activation}.
For an explicit $f$, Theorem~\ref{thm:theta-normalized-DR} can be used to check a proposed pair.
If $f$ is $C^2$ and $\theta$ is differentiable on $[0,1]^d \setminus \{\BF{0}\}$, then
$\partial_j ( \partial_i f / \theta ) \le 0$ on $[0,1]^d \setminus \{\BF{0}\}$ for all $i, j$
implies~\eqref{eq:normalized-dr} with $\gamma = 1$, because $\nabla f / \theta$ is then coordinate-wise nonincreasing along every segment from $\x$ to $\y$ with $\y \geq \x \neq \BF{0}$.
More generally, any lower bound $\gamma \in (0, 1]$ on the ratios $(\partial_i f(\x) / \theta(\x)) / (\partial_i f(\y) / \theta(\y))$ over all $i$ and $\y \geq \x \neq \BF{0}$ with $\partial_i f(\y) > 0$ satisfies~\eqref{eq:normalized-dr}.
For example, $\gamma$-weakly DR-submodular functions satisfy~\eqref{eq:normalized-dr} with constant $\theta$, and nonnegative monotone $\sigma$-OSS functions are $1$-weakly $\theta$-up-concave with $\theta(\x) = \|\x\|_1^{2\sigma}$ (Theorem~\ref{thm:oss-equivalence}).

\begin{figure}[t]
\centering
\resizebox{\textwidth}{!}{%
\begin{tikzpicture}[
    every node/.style={rectangle, draw, rounded corners, align=center, minimum height=0.8cm, inner sep=4pt, font=\large},
    main/.style={fill=blue!10, font=\large\bfseries},
    subclass/.style={fill=yellow!10},
    subset/.style={->, thick, >=stealth}
]
\node[subclass] (A) at (0,0) {$C^2 \cap \mathcal{F}_{1}^{\mathrm{DR}}$};
\node[subclass] (B) at (4.2,0) {$\mathcal{F}_{\gamma}^{\mathrm{DR}}$};
\node[subclass] (C) at (8,0) {$\mathcal{F}_{\gamma}$};
\node[subclass] (D) at (12.5,0) {$\mathcal{F}_{\gamma,1,\sigma}$};
\node[main] (E) at (16.5,0) {$\mathcal{F}_{\gamma,\theta}$};
\node[subclass] (F) at (0,-2.2) {$\mathcal{F}_{0}^{\mathrm{OSS}} = C^2 \cap \mathcal{F}_{1}$};
\node[subclass] (G) at (5.6,-2.2) {$C^2 \cap \mathcal{F}_{1,1,\sigma/2}$};
\node[subclass] (H) at (9.6,-2.2) {$\mathcal{F}_{\sigma/2}^{\mathrm{OSS}}$};
\node[subclass] (I) at (12.5,-2.2) {$\mathcal{F}_{1,1,\sigma}$};
\draw[subset] (A) -- (B);
\draw[subset] (B) -- (C);
\draw[subset] (C) -- (D);
\draw[subset] (D) -- (E);
\draw[subset] (A) -- (F);
\draw[subset] (F) -- (G);
\draw[subset] (G) -- (H);
\draw[subset] (H) -- (I);
\draw[subset] (I) -- (D);
\draw[subset] (F.north east) -- (C.south west);
\end{tikzpicture}
}
\caption{\small Containment relations among the function classes introduced in Section~\ref{sec:setup}.
An arrow $A \to B$ indicates that $A \subseteq B$.
The arrow into $\mathcal{F}_{\gamma,\theta}$ means that $\mathcal{F}_{\gamma,1,\sigma}$ is the special case $\theta(\x) = \|\x\|_1^\sigma$.}
\label{fig:class-hierarchy}
\end{figure}
\paragraph{Relation to Existing Classes.}

For constant $\theta$, that is, for $\C{F}_{\gamma, p, 0}$, condition~\eqref{eq:weak-theta} is $\gamma$-weak up-concavity~\cite{zhang22_stoch_contin_submod_maxim,pedramfar2024linear}: for all $\x \leq \y$ in $[0,1]^d$,
\begin{align}
\gamma \bra \nabla f(\y), \y - \x \ket
\leq f(\y) - f(\x)
\leq \tfrac{1}{\gamma} \bra \nabla f(\x), \y - \x \ket.
\end{align}

We write $\C{F}_{\gamma}$ for the class of nonnegative monotone $\gamma$-weakly up-concave functions. It contains the $\gamma$-weakly DR-submodular functions~\cite{hassani17_gradien_method_submod_maxim,PedramfarQuinnAggarwal2024}, defined by
\begin{equation}
\label{eq:weak-DR}
\gamma \nabla f(\y) \leq \nabla f(\x),
\end{equation}
for all $\x,\y \in [0,1]^d$ with $\y \geq \x$.
This is~\eqref{eq:normalized-dr} with constant $\theta$; ordinary DR-submodularity is the case $\gamma = 1$. We write $\C{F}_{\gamma}^{\op{DR}}$ for the class of nonnegative monotone $\gamma$-weakly DR-submodular functions.

For $\sigma \ge 0$, a $C^2$ function $f$ is \emph{$\sigma$-One-Sided Smooth ($\sigma$-OSS)}~\cite{ghadiri2025beyond,zhang2024online} if
\begin{equation}
\label{eq:sigma-oss}
\frac{1}{2} \uu^\top \nabla^2 f(\x) \uu \leq \sigma \frac{\|\uu\|_1}{\|\x\|_1} \uu^\top \nabla f(\x)
\end{equation}
for all $\x \in [0,1]^d \setminus \{\BF{0}\}$ and $\uu \geq \BF{0}$.
We write $\C{F}_{\sigma}^{\op{OSS}}$ for the class of nonnegative monotone $\sigma$-OSS functions, as in~\cite{ghadiri2025beyond}. Unlike DR-submodularity, the $\sigma$-OSS condition allows marginal gains to increase (Lemma~\ref{lem:oss-acc-then-dim-returns}). The next theorem, proved in Appendix~\ref{apx:oss-equivalence}, relates it to $(1,\sigma)$-up-concavity.

\begin{theorem}\label{thm:oss-equivalence}
For any $\sigma \geq 0$, we have $C^2 \cap \C{F}_{1, 1, \sigma} \subseteq \C{F}_{\sigma}^{\op{OSS}} \subseteq \C{F}_{1, 1, 2\sigma}$.
\end{theorem}
For $\sigma > 0$, the exponent $2\sigma$ cannot be reduced: for every $\beta \in [0, 2\sigma)$, some nonnegative monotone $\sigma$-OSS function is not in $\C{F}_{1,1,\beta}$ (Proposition~\ref{prop:oss-sharp} in Appendix~\ref{apx:oss-equivalence}).

Figure~\ref{fig:class-hierarchy} shows the resulting inclusions, using also that $\mathcal{F}_{0}^{\mathrm{OSS}}$ contains $C^2 \cap \C{F}_{1}^{\op{DR}}$~\cite{ghadiri2025beyond}. Appendix~\ref{apx:beyondoss} constructs functions that are $1$-weakly $\theta$-up-concave but neither OSS nor $\gamma$-weakly $(p,\sigma)$-up-concave for any parameters; they are flat at the origin and can have accumulating-then-diminishing returns (Lemma~\ref{result_beyondoss}). The following example is of this kind.

\begin{example}[Allocation with slow start and saturation]\label{ex:activation}
Let $x_i$ be the amount allocated to channel $i$ (e.g., advertising exposure), let $\BF{a} \in \B{R}^d_{>0}$ describe the effective exposure produced by each channel, and let $M > 0$.
Define
\[
f(\x) = \int_0^{\BF{a}^\top \x} \exp\Big(-\frac{1}{t} - M t\Big)\, dt,
\qquad
\theta(\x) = \exp\Big(-\frac{1}{\BF{a}^\top \x}\Big) \text{ for } \x \neq \BF{0},
\qquad \theta(\BF{0}) = 0.
\]
Then $\nabla f(\x) / \theta(\x) = e^{-M \BF{a}^\top \x} \BF{a}$ is coordinate-wise nonincreasing, so $f$ is $1$-weakly $\theta$-up-concave by Theorem~\ref{thm:theta-normalized-DR}.
As a function of the effective exposure $s = \BF{a}^\top \x$, the marginal gain $e^{-1/s - Ms}$ is exponentially small for small $s$, increases for $s < 1/\sqrt{M}$, and decreases for $s > 1/\sqrt{M}$; both phases occur on $[0,1]^d$ when $M > \|\BF{a}\|_1^{-2}$.
Since $f$ is flat at the origin, it is neither $\sigma$-OSS nor $\gamma$-weakly $(p,\sigma)$-up-concave for any parameters (Lemma~\ref{lem:activation} in Appendix~\ref{apx:beyondoss}).

\end{example}

\section{\texorpdfstring{$\gamma$}{gamma}-Weakly \texorpdfstring{$\theta$}{theta}-Up-Concave Functions are Upper-Linearizable}\label{sec:linear}

\begin{wraptable}{r}{0.52\textwidth}
\vspace{-.15in}
\centering
{\footnotesize
\setlength{\tabcolsep}{3pt}
\begin{tabular}{c|c}
\hline
$\C{K}$ & $\C{K}^*$ \\
\hline
$[0,1]^d$ & $\{\mathbf{1}\}$ \\
$\{\x \in [0,1]^d : \|\x\|_1 \le 1\}$ & $\{\x \in [0,1]^d : \|\x\|_1 = 1\}$ \\
Matroid independence polytope & Matroid basis polytope \\
$\{\x \in [0,1]^d : \BF{a}^\top \x \le B\}$ & $\{\x \in [0,1]^d : \BF{a}^\top \x = B\}$ \\
\hline
\end{tabular}
}
\vspace{-.05in}
\caption{\small Examples of $\C{K}$ and corresponding $\C{K}^*$ (last row: $\BF{a} > \BF{0}$, $0 < B \le \sum_i a_i$).}\label{ex:KK}
\vspace{-.15in}
\end{wraptable}
Our main result, Theorem~\ref{thm:linearizable}, shows that $\gamma$-weakly $\theta$-up-concave functions are upper-linearizable. It uses the following quantities.

\begin{definition}\label{def:r}
Let $\mathcal{K} \subseteq [0,1]^d$ be a nonempty, closed, convex set. 
We define the \emph{maximal subset} $\C{K}^{\op{m}} \subseteq \C{K}$ to be the set of points $\vv \in \C{K}$ such that there is no $\uu \in \C{K}$ with $\uu \geq \vv$ and $\uu \neq \vv$.
We define $\C{K}^* = \overline{\op{conv}}(\C{K}^{\op m})$, the closed convex hull of $\C{K}^{\op m}$.
For a set $A \subseteq \mathbb{R}^d$, $\op{conv}(A)$ is its convex hull and $\overline{\op{conv}}(A)$ the closure of the convex hull. For $\uu, \vv \in \B{R}^d$, $\uu \vee \vv$ is the coordinate-wise maximum.
For a continuous monotone function $\theta : [0, 1]^d \to \B{R}_{\ge 0}$, we define 
\begin{align*}
r_\theta(\C{K}) &= \min_{\vv \in \C{K}^*} \theta(\vv)
,\quad
R_\theta(\C{K}) = \max_{\vv, \uu \in \C{K}} \theta(\vv \vee \uu) = \max_{\vv, \uu \in \C{K}^*} \theta(\vv \vee \uu).
\end{align*}
For $\theta(\vv) = \|\vv\|_p$ with $p \geq 1$, we write $r_p$ and $R_p$.
\end{definition}

For $\C{K} = [0, 1]^d$, we have $\C{K}^* = \C{K}^{\op{m}} = \{\BF{1}\}$, so $r_\theta = R_\theta$ for every $\theta$; Table~\ref{ex:KK} lists further examples.

Every point of the compact set $\C{K}$ is dominated by a point of $\C{K}^{\op m}$, so a continuous monotone function on $\C{K}$ attains its maximum on $\C{K}^{\op m} \subseteq \C{K}^* \subseteq \C{K}$, and the optimization may be restricted to $\C{K}^*$.

\begin{theorem}\label{thm:linearizable}
Let $\C{K} \subseteq [0, 1]^d$ be a closed convex set containing the origin and at least one other point, and let $\C{F} = \C{F}_{\gamma, \theta}$ be the class of $\gamma$-weakly $\theta$-up-concave functions.
Let $R_\theta = R_\theta(\C{K})$ and 
$l(r, \x) := \frac{-\gamma}{R_\theta} \int_r^1 \theta(s \x) ds$, and 
let $\F{g} : \C{F} \times \C{K} \to \B{R}^d$ be the functional defined by
\begin{align*}
\F{g}(f, \x)
&= \int_0^1 e^{l(r, \x)} \nabla f(r \x) dr,
\end{align*}
for all $f \in \C{F}$ and $\x \in \C{K}$.
Then $
\bra \F{g}(f, \x), \y - \x \ket 
\geq \alpha_\x f(\y) - f(\x),$
for all $\y \in \C{K}$ and $\x \in \C{K}$, where 
$
\alpha_\x = 1 - \exp\left( - \frac{\gamma}{R_\theta} \int_0^1 \theta(s \x) ds \right).$
\end{theorem}

\begin{proof}
Throughout the proof, fix $f \in \C{F}$ and $\x, \y \in \C{K}$, and write $R = R_\theta$. Note that $R > 0$, since $\C{K}$ contains a point $\vv \neq \BF{0}$ and $R \geq \theta(\vv \vee \vv) = \theta(\vv) > 0$.

\emph{Preliminaries on $l$.}
Recall that $l(r, \x) = -\frac{\gamma}{R} \int_r^1 \theta(s\x) ds$ for $r \in [0,1]$. Since $s \mapsto \theta(s\x)$ is continuous and nonnegative, the fundamental theorem of calculus gives
\begin{equation}\label{eq:pF-l}
l(1, \x) = 0,
\qquad
l(r, \x) \leq 0,
\qquad
\frac{\partial}{\partial r} l(r, \x) = \frac{\gamma}{R} \theta(r\x),
\qquad
\frac{d}{dr} e^{l(r, \x)} = \frac{\gamma}{R} \theta(r\x) e^{l(r, \x)} .
\end{equation}
Integrating the last identity in~\eqref{eq:pF-l} over $r \in [0,1]$ and using $l(1, \x) = 0$, we obtain
\begin{equation}\label{eq:pF-alpha}
\int_0^1 \frac{\gamma}{R} \theta(r\x) e^{l(r, \x)} dr
= e^{l(1, \x)} - e^{l(0, \x)}
= 1 - \exp\left( - \frac{\gamma}{R} \int_0^1 \theta(s\x) ds \right)
= \alpha_\x .
\end{equation}
We also introduce the shorthand
\begin{equation}\label{eq:pF-J}
J := \int_0^1 \frac{\gamma}{R} \theta(r\x) e^{l(r, \x)} f(r\x) dr .
\end{equation}

\emph{An upper bound on $\bra \F{g}(f, \x), \x \ket$.}
By the definition of $\F{g}$ and linearity of the integral,
\begin{equation}\label{eq:pF-gx}
\bra \F{g}(f, \x), \x \ket = \int_0^1 e^{l(r, \x)} \bra \nabla f(r\x), \x \ket dr = \int_0^1 e^{l(r, \x)} \frac{d}{dr} f(r\x) \, dr ,
\end{equation}
where we used the chain rule $\frac{d}{dr} f(r\x) = \bra \nabla f(r\x), \x \ket$. Both $r \mapsto e^{l(r, \x)}$ and $r \mapsto f(r\x)$ are $C^1$ on $[0,1]$, so we may integrate~\eqref{eq:pF-gx} by parts. Using~\eqref{eq:pF-l} and~\eqref{eq:pF-J},
\begin{equation}\label{eq:pF-ibp}
\bra \F{g}(f, \x), \x \ket
= e^{l(1, \x)} f(\x) - e^{l(0, \x)} f(\BF{0}) - \int_0^1 f(r\x) \frac{d}{dr} e^{l(r, \x)} dr
= f(\x) - e^{l(0, \x)} f(\BF{0}) - J .
\end{equation}
Since $f(\BF{0}) \geq 0$ and $e^{l(0, \x)} > 0$, \eqref{eq:pF-ibp} implies
\begin{equation}\label{eq:pF-upper}
\bra \F{g}(f, \x), \x \ket \leq f(\x) - J .
\end{equation}

\emph{A pointwise lower bound.}
We claim that, for every $r \in [0,1]$,
\begin{equation}\label{eq:pF-claim}
\bra \nabla f(r\x), \y \ket \geq \frac{\gamma}{R} \theta(r\x) \big( f(\y) - f(r\x) \big) .
\end{equation}
Fix $r \in [0,1]$ and let $\BF{w} = \y \vee (r\x)$ be the coordinate-wise maximum. Since $\C{K}$ is convex and contains $\BF{0}$ and $\x$, it contains $r\x$, and therefore, by the definition of $R_\theta$,
\begin{equation}\label{eq:pF-R}
\theta(\BF{w}) = \theta(\y \vee (r\x)) \leq R .
\end{equation}
For each coordinate $i$, if $y_i \geq r x_i$ then $w_i - r x_i = y_i - r x_i \in [0, y_i]$, and if $y_i < r x_i$ then $w_i - r x_i = 0$. Hence
\begin{equation}\label{eq:pF-sand}
\BF{0} \leq \BF{w} - r\x \leq \y .
\end{equation}
Since $f$ is monotone, $\nabla f(r\x) \geq \BF{0}$. Together with~\eqref{eq:pF-sand}, this gives
\begin{equation}\label{eq:pF-step1}
\bra \nabla f(r\x), \y \ket \geq \bra \nabla f(r\x), \BF{w} - r\x \ket .
\end{equation}
Now we consider two cases.

\noindent\textit{Case $r\x \neq \BF{0}$.} The pair $(r\x, \BF{w})$ satisfies $\BF{w} \geq r\x \neq \BF{0}$, and $\theta(\BF{w}) \geq \theta(r\x) > 0$. The upper inequality in~\eqref{eq:weak-theta} for this pair reads $f(\BF{w}) - f(r\x) \leq \frac{\theta(\BF{w})}{\gamma \theta(r\x)} \bra \nabla f(r\x), \BF{w} - r\x \ket$, that is,
\begin{equation}\label{eq:pF-step2}
\bra \nabla f(r\x), \BF{w} - r\x \ket \geq \frac{\gamma \theta(r\x)}{\theta(\BF{w})} \big( f(\BF{w}) - f(r\x) \big) .
\end{equation}
Since $\BF{w} \geq r\x$ and $f$ is monotone, $f(\BF{w}) - f(r\x) \geq 0$. Hence, by~\eqref{eq:pF-R}, replacing $\theta(\BF{w})$ by the larger number $R$ can only decrease the right-hand side of~\eqref{eq:pF-step2}. Moreover, $\BF{w} \geq \y$, so $f(\BF{w}) \geq f(\y)$. Combining these facts with~\eqref{eq:pF-step1} and~\eqref{eq:pF-step2},
\[
\bra \nabla f(r\x), \y \ket
\geq \frac{\gamma \theta(r\x)}{\theta(\BF{w})} \big( f(\BF{w}) - f(r\x) \big)
\geq \frac{\gamma \theta(r\x)}{R} \big( f(\BF{w}) - f(r\x) \big)
\geq \frac{\gamma \theta(r\x)}{R} \big( f(\y) - f(r\x) \big),
\]
which is~\eqref{eq:pF-claim}.

\noindent\textit{Case $r\x = \BF{0}$.} Then $\BF{w} = \y$. If $\theta(\BF{0}) = 0$, the right-hand side of~\eqref{eq:pF-claim} is $0$, and the left-hand side is nonnegative because $\nabla f(\BF{0}) \geq \BF{0}$ and $\y \geq \BF{0}$. If $\theta(\BF{0}) > 0$ and $\y \neq \BF{0}$, then~\eqref{eq:weak-theta} holds for the pair $(\BF{0}, \y)$ by continuity (see Section~\ref{sec:setup}), and the argument of the previous case applies verbatim. If $\y = \BF{0}$, both sides of~\eqref{eq:pF-claim} are $0$.

\emph{Integrating the lower bound.}
Multiplying~\eqref{eq:pF-claim} by $e^{l(r, \x)} > 0$ and integrating over $r \in [0,1]$, we get
\begin{equation}\label{eq:pF-gy}
\bra \F{g}(f, \x), \y \ket
= \int_0^1 e^{l(r, \x)} \bra \nabla f(r\x), \y \ket dr
\geq f(\y) \int_0^1 \frac{\gamma}{R} \theta(r\x) e^{l(r, \x)} dr - J
= \alpha_\x f(\y) - J,
\end{equation}
where the last equality uses~\eqref{eq:pF-alpha} and~\eqref{eq:pF-J}.

\emph{Combining the bounds.}
Subtracting~\eqref{eq:pF-upper} from~\eqref{eq:pF-gy}, the term $J$ cancels and we obtain
\[
\bra \F{g}(f, \x), \y - \x \ket
= \bra \F{g}(f, \x), \y \ket - \bra \F{g}(f, \x), \x \ket
\geq \big( \alpha_\x f(\y) - J \big) - \big( f(\x) - J \big)
= \alpha_\x f(\y) - f(\x). \qedhere
\]
\end{proof}

The surrogate gradient $\F{g}(f, \x)$ uses $\nabla f$ only on the segment from $\BF{0}$ to $\x$, which lies in $\C{K}$ because $\BF{0} \in \C{K}$. The proof uses $R_\theta$ only through $\theta(\y \vee (r\x)) \leq R_\theta$, so $R_\theta$ may be replaced throughout (in $l$, $\F{g}$, the query algorithm and $\alpha$) by any known upper bound $\overline{R}_\theta \geq R_\theta(\C{K})$, for example $\overline{R}_\theta = \theta(\BF{1})$.

When $\inf_{\x\in \C{K}}\alpha_\x > 0$, this infimum gives a uniform approximation coefficient on $\C{K}$.
\begin{corollary}
If $\theta \equiv 1$, then $\inf_{\x\in \C{K}}\alpha_\x = 1 - e^{-\gamma}$, the best known coefficient for $\gamma$-weakly up-concave functions~\cite{pedramfar2024linear,zhang22_stoch_contin_submod_maxim}.
\end{corollary}
\begin{proof}
If $\theta \equiv 1$, then $R_\theta = \max_{\vv, \uu \in \C{K}} \theta(\vv \vee \uu) = 1$ and $\int_0^1 \theta(s\x) ds = 1$ for every $\x \in \C{K}$, so $\alpha_\x = 1 - e^{-\gamma}$ for every $\x \in \C{K}$.
\end{proof}

If $\theta(\BF{0}) = 0$, then $\alpha_{\BF{0}} = 0$, so the infimum over $\C{K}$ is $0$. Restricting the actions to a subset of $\C{K}$ that is bounded away from the origin keeps $\alpha_\x$ bounded away from $0$.

\begin{corollary}\label{cor:main-alpha}
Under the assumptions of Theorem~\ref{thm:linearizable}, we have $\BF{0} \notin \C{K}^*$ and
\[
\alpha := 1 - \exp\left( - \frac{\gamma}{R_\theta} \inf_{\x \in \C{K}^*} \int_0^1 \theta(s \x) ds \right) > 0.
\]
Moreover, for all $f \in \C F$ and $\x, \y \in \C{K}^*$,
\[
\bra \F{g}(f, \x), \y - \x \ket \geq \alpha f(\y) - f(\x).
\]
The proof is in Appendix~\ref{apx:kstar-away}.
\end{corollary}
For $\theta(\z) = \|\z\|_p^\sigma$ with $p \geq 1$ and $\sigma \geq 0$, we have $R_\theta(\C{K}) = (R_p(\C{K}))^\sigma$, $r_\theta(\C{K}) = (r_p(\C{K}))^\sigma$, and
$\alpha 
= 1 - \exp\left( - \frac{\gamma}{\sigma + 1} \left(\frac{r_p}{R_p}\right)^\sigma \right)$.

Section~\ref{sec:main-alg} shows that this inequality on $\C{K}^*$ suffices to obtain these coefficients for maximization over $\C{K}$.

\section{Main Algorithm}
\label{sec:main-alg}

We now combine Theorem~\ref{thm:linearizable} with the meta-algorithm $\mathtt{OMBQ}$ of~\cite{pedramfar2024linear} (Section~\ref{sec:prelim-linearizable}). For our class, $\mathtt{OMBQ}$ uses the surrogate gradient $\F{g}$ of Theorem~\ref{thm:linearizable} and the following query algorithm, which estimates $\F{g}$ without bias from a single stochastic gradient.

\begin{definition}\label{def:G}
Let $\C{K} \subseteq [0,1]^d$ be a nonempty, closed, convex set that contains the origin and at least one other point.
For $\x \in \C{K}$, let $l(r, \x) := \frac{-\gamma}{R_\theta} \int_r^1 \theta(s \x) ds$ as in Theorem~\ref{thm:linearizable}.
Given a stochastic first-order query oracle $\C{Q}_f$ for $f$, we define
$\C{G}(\C{Q}_f, \x) := e^{l(Z, \x)} \C{Q}_f(Z \x)$, where $Z$ is sampled uniformly from $[0, 1]$ using fresh randomization.
For $\theta(\x) = \|\x\|_p^\sigma$, the weight is explicit: $l(Z, \x) = -\frac{\gamma \|\x\|_p^\sigma}{(\sigma+1) R_\theta} (1 - Z^{\sigma+1})$. For a general $\theta$, we assume that this one-dimensional integral can be evaluated.
\end{definition}

The base algorithm plays only points of $\C{K}^*$. On $\C{K}^*$, the point-dependent inequality of Theorem~\ref{thm:linearizable} becomes the uniform inequality of Corollary~\ref{cor:main-alpha}, which gives the following regret bound, proved in Appendix~\ref{apx:regret_main}.

\begin{theorem}
\label{thm:regret_main}
Let $\C{K} \subseteq [0, 1]^d$ be a closed convex set containing the origin and at least one other point, and let $\C{F} = \C{F}_{\gamma, \theta}$ be the class of $\gamma$-weakly $\theta$-up-concave functions.
Also let $\C{G}$ be as above and let $h = \op{Id}_{\C{K}^*}$.
Then $\C{G}$ is a query algorithm for $\F{g}$, which is bounded by $B_1$ whenever $\mathcal Q_f$ is bounded by $B_1 \geq 0$.
Moreover, if $\mathcal{A}(\C{K}^*)$ is a deterministic algorithm for online optimization with semi-bandit feedback over $\C{K}^*$ and 
 $\mathcal{A}' = \mathtt{OMBQ}(\mathcal{A}(\C{K}^*), \mathcal{G}, h)$, then 
\begin{equation*}
\mathcal{R}^{\mathcal{A}'}_{\alpha, \op{Adv}^{\op{o}}_1(\C{F}, B_1)}
\le 
\mathcal{R}^{\mathcal{A}(\C{K}^*)}_{1, \op{Adv}_1(\C{L}(\C{K}^*)[B_1])},
\end{equation*}
where $
\alpha = 1 - \exp\left( - \frac{\gamma}{R_\theta} \inf_{\x \in \C{K}^*} \int_0^1 \theta(s \x) ds \right)$.
The inequality holds for every interval and every comparator set $\C{U} \subseteq (\C{K}^*)^T$, used on both sides. For static and adaptive regret, the nonlinear benchmark on the left may equivalently be maximized over $\C{K}$, while the linear comparator domain on the right remains $\C{K}^*$.
\end{theorem}

\begin{remark}\label{rem:solvable}
The base algorithm needs a set oracle for $\C{K}^*$: a projection oracle for $\mathtt{OGA}$, or a separation oracle for the Separation-Oracle OGA ($\mathtt{SO\textrm{-}OGA}$) of Section~\ref{sec:extensions}. Computing $\C{K}^*$ is a problem in vector optimization: if $\C{K}$ is a rational polytope, the classical Benson algorithm~\cite{benson98_outer_approx_algor_gener_all,lohne11_vector_optim_infim_suprem} describes $\C{K}^*$ in finite time, possibly exponential in the description size of $\C{K}$ (Appendix~\ref{apx:benson}). Alternatively, the algorithm may run on any closed convex set $\C{K}'$ with $\C{K}^* \subseteq \C{K}' \subseteq \C{K}$ for which the required oracle is available; Proposition~\ref{prop:outer-approx} gives the resulting coefficient, which can be zero, for example for $\C{K}' = \C{K}$ when $\theta(\BF{0}) = 0$. For matroid independence polytopes, $\C{K}^*$ is the basis polytope (Section~\ref{sec:matroid}).
\end{remark}

The next proposition is proved in Appendix~\ref{apx:outer-approx}.

\begin{proposition}\label{prop:outer-approx}
Let $\C{K} \subseteq [0, 1]^d$ be a closed convex set containing the origin and at least one other point, let $\C{F} = \C{F}_{\gamma, \theta}$, let $R_\theta = R_\theta(\C{K})$, and let $\C{K}'$ be a closed convex set with $\C{K}^* \subseteq \C{K}' \subseteq \C{K}$.
Define
\[
\alpha(\C{K}') := 1 - \exp\left( - \frac{\gamma}{R_\theta} \inf_{\x \in \C{K}'} \int_0^1 \theta(s \x) ds \right).
\]
Then the inequality of Corollary~\ref{cor:main-alpha} holds with $\C{K}'$ and $\alpha(\C{K}')$ in place of $\C{K}^*$ and $\alpha$. If $\alpha(\C{K}') > 0$, the same is true for Theorem~\ref{thm:regret_main}; if $\alpha(\C{K}') = 0$, there is no approximation guarantee. Moreover, $\alpha(\C{K}') \leq \alpha(\C{K}^*)$.
If $\theta(\x) = \|\x\|_p^\sigma$, then
$\alpha(\C{K}') = 1 - \exp\big( - \frac{\gamma}{\sigma + 1} \big( \min_{\x \in \C{K}'} \|\x\|_p / R_p(\C{K}) \big)^\sigma \big)$, which equals $1 - e^{-\gamma}$ when $\sigma = 0$.
If $\theta$ is $L$-Lipschitz with respect to a norm $\| \cdot \|$ and every point of $\C{K}'$ is within distance $\varepsilon$ of $\C{K}^*$ in this norm, then
$\max\{0,\, 1 - (1 - \alpha(\C{K}^*)) \exp( \frac{\gamma L \varepsilon}{2 R_\theta} ) \}
\leq \alpha(\C{K}')
\leq \alpha(\C{K}^*)$.

\end{proposition}

\begin{wrapfigure}{r}{0.42\textwidth}
\vspace{-\baselineskip}
\begin{minipage}{0.42\textwidth}
\small

\begin{algorithm2e}[H]
\caption{\small $\mathtt{OMBQ}$ with Online Gradient Ascent}
\label{alg:ombq-oga}
\SetKwInOut{Input}{Input}\DontPrintSemicolon
\Input{Oracles $\C{Q}_{f_t}$, projection $\C{P}_{\C{K}^*}$, step sizes $(\eta_t)$.}
Let $\x_1$ be a point in $\C{K}^*$. \;
\For{$t = 1, 2, \dots, T$}{
    Play $\x_t$. \;
    Sample $Z_t \sim \op{Unif}[0,1]$; $\BF{o}_t \gets \C{Q}_{f_t}(Z_t \x_t)$. \;
    $\BF{g}_t \gets e^{l(Z_t, \x_t)} \BF{o}_t$. \;
    $\x_{t+1} \gets \C{P}_{\C{K}^*}(\x_t + \eta_t \BF{g}_t)$ \;
}
\end{algorithm2e}
\end{minipage}
\vspace{-\baselineskip}
\end{wrapfigure}
Since every action lies in $\C{K}^*$, the coefficient $\alpha_\x$ of Theorem~\ref{thm:linearizable} at each played point is at least the $\alpha$ of Corollary~\ref{cor:main-alpha}. Algorithm~\ref{alg:ombq-oga} gives the pseudo-code when the base algorithm is Online Gradient Ascent ($\mathtt{OGA}$); Section~\ref{sec:extensions} uses other base algorithms.

\leavevmode\textbf{Feedback model.} The base algorithm $\C{A}(\C{K}^*)$ receives $\BF{g}_t$ as the gradient of a linear reward at its own action $\x_t$, so it is semi-bandit for the surrogate linear problem.
The composed algorithm $\C{A}' = \mathtt{OMBQ}(\C{A}(\C{K}^*), \C{G}, \op{Id})$ plays $\x_t \in \C{K}^*$ but queries $\C{Q}_{f_t}$ at $Z_t \x_t$ with $Z_t$ uniform on $[0,1]$, which in general differs from $\x_t$.
Hence $\C{A}'$ is a first-order full-information algorithm for the original problem, and $f_t$ and $\C{Q}_{f_t}$ must be defined on all of $\C{K}$.
Section~\ref{sec:extensions} obtains a semi-bandit algorithm by applying the conversion $\mathtt{SFTT}$ of~\cite{pedramfar2024linear} to $\C{A}'$.

\begin{corollary}
Under the assumptions of Theorem~\ref{thm:regret_main}, with $\alpha$ as there, assume that $\C{Q}_f$ is a stochastic first-order oracle bounded by $B_1 > 0$ for any $f \in \C F$.
With step sizes $\eta_t=\tfrac{D}{B_1 \sqrt{t}}$, where $D \geq \op{diam}(\C{K}^*)$, the static regret of Algorithm~\ref{alg:ombq-oga} satisfies
$\mathcal{R}_{\alpha, \op{Adv}^{\op{o}}_1(\C{F}, B_1)} 
\le 
\frac{3}{2} B_1 D \sqrt{T}$.
\end{corollary}

\section{\texorpdfstring{$\gamma$}{gamma}-weakly \texorpdfstring{$(1,\sigma)$}{(1,sigma)}-up-concave Functions with Matroid Constraint Sets}
\label{sec:matroid}

We recall the standard definitions.

\begin{definition}[Matroid]
A \emph{matroid} is a pair $\mathcal{M}=(V,\mathcal{I})$, where $V$ is a finite set, called the \emph{ground set} and $\mathcal{I} \subseteq 2^V$ is a non-empty family of subsets, called the \emph{independent sets}, satisfying:
\begin{enumerate}[leftmargin=0pt,labelindent=0pt,labelwidth=!,itemindent=6pt,align=left,itemsep=2pt,topsep=2pt,parsep=0pt]
\item \textbf{Downward-closed property:} If $A\in\mathcal{I}$ and $B\subseteq A$, then $B \in\mathcal{I}$.
\item \textbf{Exchange property:} If $A,B \in \mathcal{I}$ and $|A|<|B|$, then there exists an element 
$e \in B \setminus A$ such that $A\cup\{e\}\in\mathcal{I}$.
\end{enumerate}
A maximal independent set is called a \emph{basis}, and all bases have the same cardinality, called the \emph{rank} of the matroid, denoted by $\rho(\C M)$.
\end{definition}

\begin{definition}[Independence and Basis Polytopes]
Let $\mathcal{M}=(V,\mathcal{I})$ be a matroid with $V=\{1, \cdots, d\}$.
The \emph{independence polytope} and \emph{basis polytope} are defined as
\[
P_{\mathcal{I}}
:= \operatorname{conv}\{\mathbf{1}_S : S\in\mathcal{I}\}
\;\subseteq\; [0,1]^d,
\quad
P_{\mathcal{B}}
:= \operatorname{conv}\{\mathbf{1}_B : B \text{ is a basis of } \mathcal{M}\},
\]
where $\mathbf{1}_S\in\{0,1\}^d$ denotes the indicator vector of $S$.
\end{definition}

The next theorem, proved in Appendix~\ref{apx:matroid-radial}, identifies $(P_{\mathcal I})^*$ and the radii $r_1$ and $R_1$ of Definition~\ref{def:r}.

\begin{theorem}\label{thm:matroid-radial}
Let $\mathcal M$ be a nonempty matroid on ground set $V=\{1,\dots,d\}$,
and let $P_{\mathcal I}$ and $P_{\mathcal B}$ denote its independence and basis polytopes, respectively.
Then $P_{\mathcal I}$ is a closed, downward-closed convex set, and
\begin{align*}
P_{\mathcal B} 
&= (P_{\mathcal I})^{\op{m}} 
= (P_{\mathcal I})^*
= P_{\mathcal I} \cap \{ \x \in \B{R}^d \mid \| \x \|_1 = \rho(\C{M}) \},\\
r_1(P_{\mathcal I})
&= r_1(P_{\mathcal B})
= \rho(\C M)
,\qquad
R_1(P_{\mathcal I})
= R_1(P_{\mathcal B})
\leq 2\rho(\C M).
\end{align*}
\end{theorem}

Combining Theorem~\ref{thm:matroid-radial} with Corollary~\ref{cor:main-alpha} gives the following.

\begin{corollary}\label{cor:matroid-approx}
Let $\mathcal M$ be a nonempty matroid on ground set $V=\{1,\dots,d\}$ with $\rho(\C M) \geq 1$, and let $\C{K} = P_{\mathcal I}$.
Then maximizing any function in the class $\C{F}_{\gamma,1,\sigma}$ over $\C{K}$
admits both offline and online algorithms with approximation coefficient $
\alpha
=
1-\exp\!\left(-\frac{\gamma}{2^{\sigma}(\sigma+1)}\right).
$
\end{corollary}
\begin{proof}
By Theorem~\ref{thm:matroid-radial}, we have
$r_1(\C K) = \rho(\C M)$ and $R_1(\C K) \leq 2\rho(\C M)$.
Substituting $p=1$ and $\tfrac{r_1}{R_1} \geq \tfrac{1}{2}$ into the formula for $\alpha$ after Corollary~\ref{cor:main-alpha}, which is increasing in $r_1/R_1$, yields the stated approximation coefficient.
Finally, Theorem~\ref{thm:regret_main} establishes $\alpha$-regret bounds for the class $\C F_{\gamma,1,\sigma}$ over $\C K$ with this value of $\alpha$.
\end{proof}

\begin{remark}
Achieving this coefficient requires set oracles for $P_{\mathcal B}$ (see Remark~\ref{rem:solvable}). Since $P_{\mathcal B}$ is the intersection of $P_{\mathcal I}$ with a hyperplane (Theorem~\ref{thm:matroid-radial}), membership and separation oracles for $P_{\mathcal B}$ follow directly from those for $P_{\mathcal I}$, and other oracles, such as linear optimization oracles, can be obtained with the ellipsoid method. The algorithm may use $\overline{R}_\theta = (2\rho(\C M))^\sigma$ in place of $R_\theta$ (Section~\ref{sec:linear}), since $R_1(P_{\mathcal I}) \leq 2\rho(\C M)$ by Theorem~\ref{thm:matroid-radial}; this gives the coefficient of Corollary~\ref{cor:matroid-approx} without computing $R_1$.
\end{remark}

For $\sigma = 0$, this is the $1-e^{-\gamma}$ coefficient for $\mathcal{F}_{\gamma,1,0} = \mathcal{F}_{\gamma}$~\cite{pedramfar2024linear}. With Theorem~\ref{thm:oss-equivalence}, it gives the following for $\sigma$-OSS functions.
\begin{corollary}\label{cor:oss-matroid}
Let $\mathcal M$ be a nonempty matroid on ground set $V=\{1,\dots,d\}$ with $\rho(\C M) \geq 1$ and let
$\C K = P_{\mathcal I}$.
Then maximizing any nonnegative monotone $\sigma$-OSS function over $\C K$ admits both offline and online
algorithms with approximation coefficient $
\alpha
=
1-\exp\!\Big(-\frac{1}{2^{2\sigma}(2\sigma+1)}\Big).$
\end{corollary}
\begin{proof}
By Theorem~\ref{thm:oss-equivalence}, $\C{F}^{\op{OSS}}_\sigma \subseteq \C{F}_{1,1,2\sigma}$, and Corollary~\ref{cor:matroid-approx} with $\gamma=1$ and exponent $2\sigma$ gives the coefficient.
\end{proof}

For $\sigma = 0$, this is the $1-1/e$ coefficient for $0$-OSS functions; here $\mathcal{F}_{0}^{\mathrm{OSS}} = C^2 \cap \mathcal{F}_{1,1,0} = C^2 \cap \mathcal{F}_{1}$, and Corollary~\ref{cor:matroid-approx} covers all of $\mathcal{F}_{1,1,0} = \mathcal{F}_{1}$~\cite{pedramfar2024linear}. For general $\sigma$, the closest results are those of~\citet{ghadiri2025beyond} for monotone $\sigma$-OSS functions over downward-closed polytopes, with coefficient $\alpha_2(\sigma) = \sup_{t \in [0, 1)} \big( 1-\exp\big(-(1-t)\big(\frac{t}{t+1}\big)^{2 \sigma} \big)\big)$. Corollary~\ref{cor:oss-matroid} gives $\alpha_1(\sigma) = 1-\exp\!\big(-\frac{1}{2^{2\sigma}(2\sigma+1)}\big)$. Appendix~\ref{apx:alpha-comparison} shows that $\alpha_1(\sigma) \geq \alpha_2(\sigma)$, with equality only at $\sigma = 0$, and that $\alpha_1(\sigma)/\alpha_2(\sigma) \to e/2$ as $\sigma \to \infty$. This comparison uses $r_1/R_1 \geq 1/2$ for matroid polytopes (Theorem~\ref{thm:matroid-radial}). For a general downward-closed polytope $\C{K}$, Theorem~\ref{thm:oss-equivalence} and Corollary~\ref{cor:main-alpha} give $1-\exp\big(-\frac{1}{2\sigma+1}\big(\frac{r_1(\C K)}{R_1(\C K)}\big)^{2\sigma}\big)$, which, for $\sigma > 0$, is smaller than $\alpha_2(\sigma)$ when $r_1(\C K)/R_1(\C K)$ is sufficiently small. \citet{ghadiri2025beyond} also give a separate guarantee under an additional third-order derivative-sign condition; the comparison above is for general monotone $\sigma$-OSS functions. All coefficients in this section concern the continuous problem over $P_{\mathcal I}$.

\citet{zhang2024online} study online optimization of monotone OSS functions. Their Frank--Wolfe-type algorithm recovers the ratio of~\cite{ghadiri2025beyond} under an additional smoothness condition, and their online-gradient-ascent algorithm, which does not need it, has ratio at most $\alpha_3(\sigma) = \sup_{t \in (0,1)}\big(1 + (\frac{t+1}{t})^{4\sigma}\big)^{-1} = (1+2^{4\sigma})^{-1}$. We have $\alpha_1(\sigma) > \alpha_3(\sigma)$ for all $\sigma \geq 0$, and $\alpha_1(\sigma)/\alpha_3(\sigma) \to \infty$ as $\sigma \to \infty$.

\section{Results for Other Settings}
\label{sec:extensions}

This section applies the meta-algorithms of~\cite{pedramfar24_unified_framew_analy_meta_onlin_convex_optim,pedramfar2024linear} to Theorem~\ref{thm:regret_main} with existing base algorithms.
For a single fixed function, $\mathtt{OTB}$ (Theorem~8 in~\cite{pedramfar2024linear}) gives an offline algorithm, and $\mathtt{SFTT}$ converts full-information algorithms into semi-bandit ones.
We use the version of $\mathtt{SFTT}$ in Theorem~7 of~\cite{pedramfar2024linear}, whose conditions hold here: $\C{F}_{\gamma, \theta}$ is closed under convex combinations (Lemma~\ref{lem:function-class-convex-cone}), $\mathtt{OMBQ}(\C{A}, \C{G}, \op{Id})$ makes one query per round that does not depend on that round's observations, and $f$ is $B_1$-Lipschitz on $\C{K}$ when $\C{Q}_f$ is bounded by $B_1$.

We use two deterministic semi-bandit base algorithms.
Separation-Oracle Online Gradient Ascent ($\mathtt{SO\textrm{-}OGA}$) is projection-free and has optimal adaptive regret~\cite{garber22_new_projec_algor_onlin_convex,pedramfar2024linear}.
As in~\cite{pedramfar2024linear}, it runs within the affine hull of $\C{K}^*$, which may have empty interior (as for the basis polytope), and is given affine coordinates, a point $\BF{c} \in \op{relint}(\C{K}^*)$, a radius $r > 0$ with $B_r(\BF{c}) \cap \op{aff}(\C{K}^*) \subseteq \C{K}^*$, and a deterministic separation oracle; the hidden constants below depend on these.
If $\C{K}^*$ is a single point, or if $B_1 = 0$ (so that every $f$ is constant on $\C{K}$), the problem is trivial.
Improved Ader ($\mathtt{IA}$) has dynamic regret with optimal dependence on the path length~\cite{zhang18_adapt_onlin_learn_dynam_envir}.
Applying the framework to $\mathtt{SO\textrm{-}OGA}$ and $\mathtt{IA}$ gives the following bounds; the proofs follow Corollaries~7 and~8 in~\cite{pedramfar2024linear}.

\begin{theorem}\label{thm:w-ada-reg}
Under the assumptions of Theorem~\ref{thm:regret_main}, we have
$\C{AR}_{\alpha, \op{Adv}^{\op{o}}_1(\C{F}, B_1)}^{ \C{A} } \leq \C{O}( B_1 T^{1/2} )$,
where $\C{A} = \mathtt{OMBQ}(\mathtt{SO\textrm{-}OGA}, \C{G}, \op{Id})$ is a first-order full-information algorithm with a single query per round.
Moreover, $\C{AR}_{\alpha, \op{Adv}^{\op{o}}_1(\C{F}, B_1)}^{\C{A}_{\op{semi-bandit}}} \leq \C{O}( B_1 T^{2/3} )$, where $\C{A}_{\op{semi-bandit}} = \mathtt{SFTT}(\C{A})$.
Moreover, the sample complexity of the offline first-order algorithm $\mathtt{OTB}(\C{A})$ is $\C{O}(\epsilon^{-2})$; here $\mathtt{OTB}(\C{A})$ outputs an iterate $X$ chosen uniformly at random, and the guarantee is $\B{E}[f(X)] \geq \alpha \max_{\x \in \C{K}} f(\x) - \epsilon$.
\end{theorem}

\begin{theorem}\label{thm:d-reg}
Under the assumptions of Theorem~\ref{thm:regret_main}, we have
\[ \C{R}_{\alpha, \op{Adv}^{\op{o}}_1(\C{F}, B_1)}^{\C{A}}(\uu)
= \C{O}( B_1 \sqrt{T(1 + P_T(\uu))} ) \]
for every comparator sequence $\uu \in (\C{K}^*)^T$,
where $\C{A} = \mathtt{OMBQ}(\mathtt{IA}, \C{G}, \op{Id})$ is also a first-order full-information algorithm with a single query per round.
\end{theorem}

\section{Conclusion}
We introduced $\gamma$-weakly $\theta$-up-concave functions, a class that strictly contains monotone DR-submodular and OSS functions and includes functions that are flat at the origin or have accumulating-then-diminishing returns. These functions are upper-linearizable, which gives offline and online guarantees and, under matroid constraints, improved coefficients for general monotone OSS objectives. Whether these coefficients are tight is an open question.

\begin{ack}
This work was in part supported by U.S. National Science Foundation under grant CCF-2149588, IVADO and CIFAR.
\end{ack}

\bibliographystyle{apalike}
\bibliography{references}

\clearpage
\appendix
\section{Related Work}
\label{sec:related}

\paragraph{DR-submodular and up-concave maximization.}
Maximization of continuous DR-submodular functions over convex sets has been studied in offline and online settings with two main approaches. Frank--Wolfe-type methods have been adapted to stochastic and online settings~\cite{bian17_contin_dr_maxim,bian17_guaran_non_optim,mualem22_resol_approx_offlin_onlin_non,
pedramfar23_unified_approac_maxim_contin_dr_funct,chen23_contin_non_dr_maxim_down_convex_const,
mokhtari20_stoch_condit_gradien_method,hassani20_stoch_condit_gradien,
chen18_onlin_contin_submod_maxim,niazadeh21_onlin_learn_offlin_greed_algor,zhang19_onlin_contin_submod_maxim,
zhang23_onlin_learn_non_submod_maxim,pedramfar23_unified_projec_free_algor_adver}. Boosting methods maximize an auxiliary objective whose maximizers approximately maximize the original function~\cite{hassani17_gradien_method_submod_maxim,chen18_onlin_contin_submod_maxim,fazel22_fast_first_order_method_monot,
zhang22_stoch_contin_submod_maxim,wan23_bandit_multi_dr_submod_maxim,liao23_improv_projec_onlin_contin_submod_maxim,
zhang24_boost_gradien_ascen_contin_dr_maxim}.

Continuous DR-submodular functions are concave along positive directions~\cite{calinescu11_maxim_monot_submod_funct_subjec_matroid_const,bian17_guaran_non_optim}, which led \citet{wilder18_equil} to define \emph{up-concave} functions. Offline maximization of up-concave functions is studied in~\cite{lee23_non_smoot_hextb_smoot_robus_submod_maxim}. Our class contains monotone DR-submodular and up-concave functions, and our guarantees come from upper-linearization rather than from problem-specific algorithms.

\paragraph{One-sided smoothness (OSS).}
\citet{ghadiri2025beyond} introduced one-sided smoothness for diversity-type objectives and other monotone nonconcave functions over downward-closed polytopes. It strictly generalizes twice-differentiable DR-submodularity, and \citet{ghadiri2025beyond} obtain offline guarantees with variants of continuous greedy. \citet{zhang2024online} give online algorithms and regret bounds for OSS objectives.
Under matroid constraints, our coefficient is strictly larger than the general OSS coefficient of~\cite{ghadiri2025beyond} for $\sigma > 0$, with asymptotic ratio $e/2$, and has better dependence on $\sigma$ than the online-gradient-ascent coefficient of~\cite{zhang2024online}. Our class strictly contains the OSS class, and our matroid coefficients need no third-order conditions; Section~\ref{sec:matroid} gives the detailed comparison. Non-monotone OSS functions are studied in~\cite{zhang2023parallelized}. Outside these structured classes, most results in continuous nonconcave optimization concern first- or second-order stationarity rather than global approximation ratios~\cite{bedi2021escaping,dixit2024gradient,beznosikov2025distributed}.

\paragraph{Upper-linearizable framework.}
\citet{pedramfar2024linear} introduced upper-linearizable functions and meta-algorithms that convert algorithms for linear optimization into algorithms for nonconcave objectives, including DR-submodular and up-concave functions. The framework transfers offline approximation guarantees and static, dynamic, and adaptive regret bounds, under first-order, zeroth-order, semi-bandit, and bandit feedback. Later work gave uniform wrapper constructions for regret conversions, improved guarantees under limited feedback, and decentralized and projection-free versions~\cite{pedramfar2025uniform,lu2025decentralized,lu2026upper}. We show that a broader class of monotone objectives is upper-linearizable, so the first-order consequences (full-information and semi-bandit) apply to it. The zeroth-order and bandit conversions of~\cite{pedramfar2024linear} also require the smoothed functions to remain in the class, which we have not shown for general $\theta$.

\section{Proof of Theorem~\ref{thm:theta-normalized-DR}}
\label{apx:theta-normalized-DR}

\begin{proof}
Fix $\x, \y \in [0,1]^d$ with $\y \geq \x \neq \BF{0}$, and define $\z_t = \x + t(\y - \x)$ for $t \in [0,1]$. Since $f$ is $C^1$, the fundamental theorem of calculus applied to $t \mapsto f(\z_t)$ gives
\begin{equation}\label{eq:pD-ftc}
f(\y) - f(\x) = \int_0^1 \bra \nabla f(\z_t), \y - \x \ket dt .
\end{equation}
Since $\x \leq \z_t \leq \y$ coordinate-wise and $\theta$ is monotone,
\begin{equation}\label{eq:pD-theta}
0 < \theta(\x) \leq \theta(\z_t) \leq \theta(\y),
\end{equation}
where $\theta(\x) > 0$ because $\x \neq \BF{0}$. In particular $\z_t \neq \BF{0}$, so~\eqref{eq:normalized-dr} may be applied to the pairs $(\z_t, \y)$ and $(\x, \z_t)$. This gives, coordinate-wise,
\begin{equation}\label{eq:pD-ndr}
\frac{\gamma \nabla f(\y)}{\theta(\y)} \leq \frac{\nabla f(\z_t)}{\theta(\z_t)}
\qquad\text{and}\qquad
\frac{\gamma \nabla f(\z_t)}{\theta(\z_t)} \leq \frac{\nabla f(\x)}{\theta(\x)} .
\end{equation}
Since $f$ is monotone, all its partial derivatives are nonnegative, so $\nabla f(\x) \geq \BF{0}$ and $\nabla f(\y) \geq \BF{0}$.
Multiplying the first inequality in~\eqref{eq:pD-ndr} by $\theta(\z_t) > 0$, and then using $\theta(\z_t) \geq \theta(\x)$ from~\eqref{eq:pD-theta} and $\nabla f(\y) \geq \BF{0}$, we get
\begin{equation}\label{eq:pD-low}
\nabla f(\z_t) \geq \frac{\gamma \theta(\z_t)}{\theta(\y)} \nabla f(\y) \geq \frac{\gamma \theta(\x)}{\theta(\y)} \nabla f(\y) .
\end{equation}
Multiplying the second inequality in~\eqref{eq:pD-ndr} by $\theta(\z_t)/\gamma > 0$, and then using $\theta(\z_t) \leq \theta(\y)$ from~\eqref{eq:pD-theta} and $\nabla f(\x) \geq \BF{0}$, we get
\begin{equation}\label{eq:pD-up}
\nabla f(\z_t) \leq \frac{\theta(\z_t)}{\gamma \theta(\x)} \nabla f(\x) \leq \frac{\theta(\y)}{\gamma \theta(\x)} \nabla f(\x) .
\end{equation}
Since $\y - \x \geq \BF{0}$, taking the inner product of the coordinate-wise inequalities~\eqref{eq:pD-low} and~\eqref{eq:pD-up} with $\y - \x$ preserves them:
\[
\frac{\gamma \theta(\x)}{\theta(\y)} \bra \nabla f(\y), \y - \x \ket
\leq \bra \nabla f(\z_t), \y - \x \ket
\leq \frac{\theta(\y)}{\gamma \theta(\x)} \bra \nabla f(\x), \y - \x \ket .
\]
The outer terms do not depend on $t$. Integrating over $t \in [0,1]$ and using~\eqref{eq:pD-ftc} gives~\eqref{eq:weak-theta}. Since $f$ is also $C^1$, nonnegative and monotone, $f$ is $\gamma$-weakly $\theta$-up-concave.
\end{proof}

\section{Proof of Theorem~\ref{thm:oss-equivalence}}\label{apx:oss-equivalence}

\begin{proof}
We prove the two inclusions separately. Throughout, we use that $\|\x + \uu\|_1 = \|\x\|_1 + \|\uu\|_1$ for $\x, \uu \geq \BF{0}$, since the $\ell_1$ norm of a nonnegative vector is the sum of its coordinates.

\emph{Proof of $C^2 \cap \C{F}_{1,1,\sigma} \subseteq \C{F}_{\sigma}^{\op{OSS}}$.}
Let $f \in C^2 \cap \C{F}_{1,1,\sigma}$. By definition, $f$ is nonnegative and monotone, so it remains to verify~\eqref{eq:sigma-oss}.
Fix $\x \in (0,1)^d$ and $\uu \geq \BF{0}$. Choose $\varepsilon > 0$ such that $\x + t\uu \in (0,1)^d$ for all $t \in [0, \varepsilon)$, and define $g(t) = f(\x + t\uu)$ for $t \in [0,\varepsilon)$. By the chain rule,
\begin{equation}\label{eq:pC-g}
g'(t) = \bra \nabla f(\x + t\uu), \uu \ket,
\qquad
g''(t) = \uu^\top \nabla^2 f(\x + t\uu) \uu .
\end{equation}
Let $h \in (0, \varepsilon)$. The pair $(\x, \x + h\uu)$ satisfies $\x + h\uu \geq \x \neq \BF{0}$, so the upper inequality in~\eqref{eq:weak-p-sigma} with $\gamma = 1$ and $p = 1$ gives
\begin{equation}\label{eq:pC-up}
g(h) - g(0)
\leq \left( \frac{\|\x + h\uu\|_1}{\|\x\|_1} \right)^\sigma \bra \nabla f(\x), h\uu \ket
= (1 + hc)^\sigma \, h \, g'(0),
\qquad\text{where } c := \frac{\|\uu\|_1}{\|\x\|_1},
\end{equation}
and we used $\|\x + h\uu\|_1 = \|\x\|_1 + h\|\uu\|_1$ and~\eqref{eq:pC-g} at $t = 0$.
Since $f$ is $C^2$, Taylor's theorem gives
\begin{equation}\label{eq:pC-taylor}
g(h) - g(0) = h g'(0) + \tfrac{1}{2} h^2 g''(0) + o(h^2)
\qquad\text{as } h \downarrow 0 .
\end{equation}
Substituting~\eqref{eq:pC-taylor} into~\eqref{eq:pC-up}, subtracting $h g'(0)$ from both sides and dividing by $h^2 > 0$, we obtain
\begin{equation}\label{eq:pC-ratio}
\tfrac{1}{2} g''(0) + \frac{o(h^2)}{h^2}
\leq \frac{(1 + hc)^\sigma - 1}{h} \, g'(0).
\end{equation}
As $h \downarrow 0$, the fraction on the right of~\eqref{eq:pC-ratio} converges to the derivative of $h \mapsto (1+hc)^\sigma$ at $h = 0$, which is $\sigma c$, and the term $o(h^2)/h^2$ converges to $0$. Taking the limit in~\eqref{eq:pC-ratio} and using~\eqref{eq:pC-g} at $t = 0$ gives
\begin{equation}\label{eq:pC-oss}
\frac{1}{2} \uu^\top \nabla^2 f(\x) \uu \leq \sigma \frac{\|\uu\|_1}{\|\x\|_1} \bra \nabla f(\x), \uu \ket ,
\end{equation}
which is~\eqref{eq:sigma-oss} for $\x \in (0,1)^d$.
Now let $\x \in [0,1]^d \setminus \{\BF{0}\}$ be arbitrary, and let $\x_k = (1 - \frac{1}{k})\x + \frac{1}{2k}\BF{1}$ for $k \geq 2$. Each coordinate of $\x_k$ lies in $[\frac{1}{2k}, 1 - \frac{1}{2k}]$, so $\x_k \in (0,1)^d$, and $\x_k \to \x$. Inequality~\eqref{eq:pC-oss} holds at every $\x_k$. Both sides of~\eqref{eq:pC-oss} are continuous in $\x$ on $[0,1]^d \setminus \{\BF{0}\}$, because $\nabla f$ and $\nabla^2 f$ are continuous and $\|\x\|_1 > 0$ there. Letting $k \to \infty$ shows that~\eqref{eq:pC-oss} holds at $\x$. Hence $f \in \C{F}_\sigma^{\op{OSS}}$.

\emph{Proof of $\C{F}_{\sigma}^{\op{OSS}} \subseteq \C{F}_{1,1,2\sigma}$.}
Let $f \in \C{F}_{\sigma}^{\op{OSS}}$. Then $f$ is nonnegative, monotone and $C^2$, so it remains to verify~\eqref{eq:weak-p-sigma} with $\gamma = 1$, $p = 1$ and exponent $2\sigma$.
Fix $\x, \y \in [0,1]^d$ with $\y \geq \x \neq \BF{0}$. For $t \in [0,1]$, define
\[
\phi(t) = \x + t(\y - \x),
\qquad
g(t) = f(\phi(t)),
\qquad
N(t) = \|\phi(t)\|_1 .
\]
Since $\x \geq \BF{0}$ and $\y - \x \geq \BF{0}$, we have
\begin{equation}\label{eq:pC-N}
N(t) = \|\x\|_1 + t \|\y - \x\|_1,
\qquad
N'(t) = \|\y - \x\|_1,
\qquad
\|\x\|_1 = N(0) \leq N(t) \leq N(1) = \|\y\|_1 ,
\end{equation}
and in particular $N(t) > 0$, so $\phi(t) \neq \BF{0}$. By the chain rule,
\begin{equation}\label{eq:pC-g2}
g'(t) = \bra \nabla f(\phi(t)), \y - \x \ket,
\qquad
g''(t) = (\y - \x)^\top \nabla^2 f(\phi(t)) (\y - \x).
\end{equation}
Since $f$ is monotone, all its partial derivatives are nonnegative, and since $\y - \x \geq \BF{0}$,~\eqref{eq:pC-g2} gives
\begin{equation}\label{eq:pC-pos}
g'(t) \geq 0 \qquad \text{for all } t \in [0,1].
\end{equation}
Applying~\eqref{eq:sigma-oss} at the point $\phi(t) \neq \BF{0}$ with $\uu = \y - \x \geq \BF{0}$, and using~\eqref{eq:pC-N} and~\eqref{eq:pC-g2}, we get
\begin{equation}\label{eq:pC-along}
g''(t) \leq 2\sigma \frac{N'(t)}{N(t)} g'(t) .
\end{equation}
Define $q(t) = g'(t) / N(t)^{2\sigma}$. By the quotient rule,
\[
q'(t)
= \frac{g''(t) N(t)^{2\sigma} - 2\sigma N(t)^{2\sigma - 1} N'(t) g'(t)}{N(t)^{4\sigma}}
= \frac{1}{N(t)^{2\sigma}} \left( g''(t) - 2\sigma \frac{N'(t)}{N(t)} g'(t) \right),
\]
and by~\eqref{eq:pC-along}, $q'(t) \leq 0$. Hence $q$ is nonincreasing on $[0,1]$, so
\begin{equation}\label{eq:pC-q}
q(1) \leq q(t) \leq q(0) \qquad \text{for all } t \in [0,1].
\end{equation}
For the upper bound, the right inequality in~\eqref{eq:pC-q}, together with $q(0) \geq 0$ (by~\eqref{eq:pC-pos}) and $N(t) \leq N(1)$ (by~\eqref{eq:pC-N}), gives
\begin{equation}\label{eq:pC-upper}
g'(t) = N(t)^{2\sigma} q(t) \leq N(t)^{2\sigma} q(0) \leq N(1)^{2\sigma} q(0) = \left( \frac{\|\y\|_1}{\|\x\|_1} \right)^{2\sigma} g'(0).
\end{equation}
For the lower bound, the left inequality in~\eqref{eq:pC-q}, together with $q(1) \geq 0$ and $N(t) \geq N(0)$, gives
\begin{equation}\label{eq:pC-lower}
g'(t) = N(t)^{2\sigma} q(t) \geq N(t)^{2\sigma} q(1) \geq N(0)^{2\sigma} q(1) = \left( \frac{\|\x\|_1}{\|\y\|_1} \right)^{2\sigma} g'(1).
\end{equation}
By the fundamental theorem of calculus, $f(\y) - f(\x) = g(1) - g(0) = \int_0^1 g'(t) dt$. Integrating~\eqref{eq:pC-upper} and~\eqref{eq:pC-lower} over $t \in [0,1]$ and using~\eqref{eq:pC-g2} at $t = 0$ and $t = 1$, we obtain
\[
\left( \frac{\|\x\|_1}{\|\y\|_1} \right)^{2\sigma} \bra \nabla f(\y), \y - \x \ket
\leq f(\y) - f(\x)
\leq \left( \frac{\|\y\|_1}{\|\x\|_1} \right)^{2\sigma} \bra \nabla f(\x), \y - \x \ket,
\]
which is~\eqref{eq:weak-p-sigma} with $\gamma = 1$, $p = 1$ and exponent $2\sigma$. Hence $f \in \C{F}_{1,1,2\sigma}$.
\end{proof}

\begin{remark}
In the proof above, the inequality $q(1) \leq q(0)$ from~\eqref{eq:pC-q}, i.e.,
\[
\frac{g'(1)}{\|\y\|_1^{2\sigma}}
\leq
\frac{g'(0)}{\|\x\|_1^{2\sigma}}
\]
directly implies that
\begin{align*}
\langle\nabla f(\y),\y-\x\rangle
\le
\Big(\frac{\|\y\|_1}{\|\x\|_1}\Big)^{2\sigma}
\langle\nabla f(\x),\y-\x\rangle.
\end{align*}
\end{remark}

\begin{proposition}\label{prop:oss-sharp}
For every $\sigma > 0$, $\beta \in [0, 2\sigma)$ and $\gamma \in (0, 1]$, there is a nonnegative, monotone, $C^2$, $\sigma$-OSS function on $[0,1]^d$ that is not $\gamma$-weakly $(1,\beta)$-up-concave.
In particular, the exponent $2\sigma$ in Theorem~\ref{thm:oss-equivalence} cannot be replaced by a smaller one.

\end{proposition}

\begin{proof}
Fix $\sigma > 0$, $\beta \in [0, 2\sigma)$ and $\gamma \in (0,1]$. For $\delta \in (0,1]$, define
\begin{equation}\label{eq:pS-f}
f_\delta(\x) = (\BF{1}^\top \x + \delta)^{1 + 2\sigma} - \delta^{1 + 2\sigma},
\qquad \x \in [0,1]^d .
\end{equation}
Write $s = \BF{1}^\top \x$; on $[0,1]^d$ we have $s = \|\x\|_1$. Since $s + \delta \geq \delta > 0$, the function $f_\delta$ is $C^2$ on $[0,1]^d$, and
\begin{equation}\label{eq:pS-grad}
\nabla f_\delta(\x) = (1 + 2\sigma)(s + \delta)^{2\sigma} \BF{1},
\qquad
\nabla^2 f_\delta(\x) = 2\sigma(1 + 2\sigma)(s + \delta)^{2\sigma - 1} \BF{1}\BF{1}^\top .
\end{equation}
Since $s \geq 0$, we have $f_\delta \geq 0$, and since $\nabla f_\delta \geq \BF{0}$, the function $f_\delta$ is monotone.

We first show that $f_\delta$ is $\sigma$-OSS. Fix $\x \neq \BF{0}$, so that $s > 0$, and $\uu \geq \BF{0}$, so that $\BF{1}^\top \uu = \|\uu\|_1$. By~\eqref{eq:pS-grad},
\begin{align}
\frac{1}{2} \uu^\top \nabla^2 f_\delta(\x) \uu
&= \sigma (1 + 2\sigma) (s + \delta)^{2\sigma - 1} \|\uu\|_1^2, \label{eq:pS-lhs} \\
\sigma \frac{\|\uu\|_1}{\|\x\|_1} \uu^\top \nabla f_\delta(\x)
&= \sigma (1 + 2\sigma) \frac{(s + \delta)^{2\sigma}}{s} \|\uu\|_1^2 . \label{eq:pS-rhs}
\end{align}
The right-hand side of~\eqref{eq:pS-lhs} equals $s/(s+\delta)$ times the right-hand side of~\eqref{eq:pS-rhs}. Since the latter is nonnegative and $0 < s/(s+\delta) \leq 1$, \eqref{eq:pS-lhs} is at most~\eqref{eq:pS-rhs}, which is~\eqref{eq:sigma-oss}. This also covers $\uu = \BF{0}$, where both sides are $0$. So $f_\delta \in \C{F}_\sigma^{\op{OSS}}$ for every $\delta \in (0,1]$.

Next, suppose that $f_\delta$ is $\gamma$-weakly $(1,\beta)$-up-concave for some $\delta \in (0,1]$. Take $\x = \delta \BF{e}_1$ and $\y = \BF{1}$, so that $\y \geq \x \neq \BF{0}$, $\|\x\|_1 = \delta$ and $\|\y\|_1 = d$. By~\eqref{eq:pS-f}, the left-hand side of the upper inequality in~\eqref{eq:weak-p-sigma} is
\begin{equation}\label{eq:pS-diff}
f_\delta(\y) - f_\delta(\x) = (d + \delta)^{1 + 2\sigma} - (2\delta)^{1 + 2\sigma},
\end{equation}
and, by~\eqref{eq:pS-grad} with $s = \delta$ and $\bra \BF{1}, \y - \x \ket = d - \delta$, its right-hand side is
\begin{equation}\label{eq:pS-bound}
\frac{1}{\gamma} \left( \frac{d}{\delta} \right)^\beta (1 + 2\sigma)(2\delta)^{2\sigma}(d - \delta)
= \frac{(1 + 2\sigma) 2^{2\sigma} d^\beta (d - \delta)}{\gamma} \, \delta^{2\sigma - \beta}.
\end{equation}
The upper inequality therefore requires~\eqref{eq:pS-diff} to be at most~\eqref{eq:pS-bound}. As $\delta \downarrow 0$, \eqref{eq:pS-diff} converges to $d^{1 + 2\sigma} > 0$, while~\eqref{eq:pS-bound} converges to $0$ because $2\sigma - \beta > 0$. Hence the inequality fails for all sufficiently small $\delta$, and for such $\delta$ the function $f_\delta$ is not $\gamma$-weakly $(1,\beta)$-up-concave.
\end{proof}

\section{\texorpdfstring{Examples Beyond DR-Submodular and OSS Functions}{Examples Beyond DR-Submodular and OSS Functions}}\label{apx:beyondoss}

A function has accumulating-then-diminishing returns if its marginal reward first increases and then decreases. Such returns already occur for $\gamma$-weakly $(1,\sigma)$-up-concave and $\sigma$-OSS functions (Lemma~\ref{lem:oss-acc-then-dim-returns}). Flatness at the origin (Section~\ref{sec:setup}) provides a separation from these classes: by Lemma~\ref{lem:flat-functions-are-not-OSS}, a nonzero function in $\C{F}_{\gamma,\theta}$ that is flat at the origin is neither $\gamma'$-weakly $(p,\sigma)$-up-concave nor $\sigma$-OSS for any admissible parameters. Lemmas~\ref{result_beyondoss} and~\ref{lem:activation} construct such functions with a scaling $\theta$ that is itself flat at the origin; these functions depend on $\x$ only through $\|\x\|_1$ or $\BF{a}^\top \x$. Proposition~\ref{prop:multivariate} gives a multivariate family whose members need not depend on $\x$ through a single linear statistic.

\begin{lemma}\label{lem:oss-acc-then-dim-returns}
Fix $\sigma \geq 1$ and $M>0$, and define
\[
    f(\x)=\int_0^{\|\x\|_1} t^\sigma e^{-Mt}\,dt .
\]
Then $f \in \mathcal F_{1,1,\sigma} \cap \mathcal F_\sigma^{\op{OSS}}$.
Moreover, $f$ exhibits monotone
accumulating-then-diminishing returns: its marginal reward as a function of
$s=\|\x\|_1$ is increasing for $s < \sigma/M$ and decreasing for $s > \sigma/M$. Both phases occur on $[0,1]^d$ when $M > \sigma/d$.
\end{lemma}

\begin{proof}
Let $\theta(\x) = \|\x\|_1^\sigma$, and define $F(s) = \int_0^s t^\sigma e^{-Mt} dt$ for $s \geq 0$, so that $f(\x) = F(\|\x\|_1)$. On $[0,1]^d$ we have $\|\x\|_1 = \BF{1}^\top \x$. By the chain rule and the fundamental theorem of calculus, with $s = \|\x\|_1$,
\begin{equation}\label{eq:pL2-grad}
\nabla f(\x) = F'(s) \BF{1} = s^\sigma e^{-Ms} \BF{1} .
\end{equation}
Thus $f$ is $C^1$, $f \geq 0$, and $f$ is monotone because $\nabla f \geq \BF{0}$. For $\x \neq \BF{0}$, \eqref{eq:pL2-grad} gives
\begin{equation}\label{eq:pL2-norm}
\frac{\nabla f(\x)}{\theta(\x)} = e^{-Ms} \BF{1}.
\end{equation}
If $\y \geq \x \neq \BF{0}$, then $\|\y\|_1 \geq \|\x\|_1$, so by~\eqref{eq:pL2-norm}, $\nabla f(\y)/\theta(\y) \leq \nabla f(\x)/\theta(\x)$ coordinate-wise. This is~\eqref{eq:normalized-dr} with $\gamma = 1$, and Theorem~\ref{thm:theta-normalized-DR} gives $f \in \C{F}_{1,1,\sigma}$.

Next, for $s > 0$ we have $F''(s) = s^{\sigma - 1} e^{-Ms} (\sigma - Ms)$. Since $\sigma \geq 1$, the factor $s^{\sigma - 1}$ is continuous on $[0, \infty)$, so $F''$ extends continuously to $s = 0$ and $F$ is $C^2$ on $[0,\infty)$. Hence $f$ is $C^2$ on $[0,1]^d$, with $\nabla^2 f(\x) = F''(\|\x\|_1) \BF{1}\BF{1}^\top$. Since $f \in C^2 \cap \C{F}_{1,1,\sigma}$, Theorem~\ref{thm:oss-equivalence} gives $f \in \C{F}_\sigma^{\op{OSS}}$.

Finally, the marginal reward along $s = \|\x\|_1$ is $F'(s) = s^\sigma e^{-Ms}$, and for $s > 0$,
\begin{equation}\label{eq:pL2-log}
\frac{d}{ds} \log F'(s) = \frac{\sigma}{s} - M .
\end{equation}
By~\eqref{eq:pL2-log}, $F'$ is increasing for $s < \sigma/M$ and decreasing for $s > \sigma/M$. On $[0,1]^d$, $s$ ranges over $[0, d]$, so both phases occur when $\sigma/M < d$, that is, when $M > \sigma/d$.
\end{proof}

Recall from Section~\ref{sec:setup} that $\phi : [0,1]^d \to \B{R}_{\geq 0}$ is flat at the origin if $\lim_{\x \to \BF{0}} \phi(\x)/\|\x\|_1^n = 0$ for all $n > 0$.

\begin{lemma}\label{lem:flat-functions-are-not-OSS}
Let \(f\in\mathcal F_{\gamma,\theta}\) for some $\gamma \in (0, 1]$ and $\theta$ and suppose that \(f\) is flat at the origin and not identically zero.
Then $\theta$ is flat at the origin.
In particular, $f$ is not $\gamma'$-weakly $(p, \sigma)$-up-concave for any \(\gamma' \in (0,1]\), $p \geq 1$ and $\sigma \geq 0$ and $f$ is not $\sigma$-OSS for any $\sigma \geq 0$.
\end{lemma}

\begin{proof}
Define $F(t) = f(t\BF{1})$ for $t \in [0,1]$. Since $f$ is monotone and $\BF{1} \geq \x$ for every $\x \in [0,1]^d$, we have $f(\BF{1}) = \max_{\x \in [0,1]^d} f(\x)$, and $f(\BF{1}) > 0$ because $f$ is nonnegative and not identically zero. Since $f$ is flat at the origin, $F(t) \to 0$ as $t \downarrow 0$. Hence there is $t_0 \in (0, 1)$ such that
\begin{equation}\label{eq:pL3-half}
f(\BF{1}) - F(t) \geq \tfrac{1}{2} f(\BF{1}) \qquad \text{for all } t \in (0, t_0].
\end{equation}
Fix $t \in (0, t_0]$. The pair $\x = t\BF{1}$, $\y = \BF{1}$ satisfies $\y \geq \x \neq \BF{0}$, so the upper inequality in~\eqref{eq:weak-theta} gives
\begin{equation}\label{eq:pL3-up}
f(\BF{1}) - F(t)
\leq \frac{\theta(\BF{1})}{\gamma \theta(t\BF{1})} \bra \nabla f(t\BF{1}), (1 - t)\BF{1} \ket
= \frac{\theta(\BF{1}) (1 - t)}{\gamma \theta(t\BF{1})} F'(t),
\end{equation}
where we used $F'(t) = \bra \nabla f(t\BF{1}), \BF{1} \ket$. Combining~\eqref{eq:pL3-half} and~\eqref{eq:pL3-up}, and using $1 - t \leq 1$, we obtain
\begin{equation}\label{eq:pL3-deriv}
F'(t) \geq c \, \theta(t\BF{1}) \quad \text{for all } t \in (0, t_0],
\qquad\text{where } c := \frac{\gamma f(\BF{1})}{2 \theta(\BF{1})} > 0 .
\end{equation}
For $t \in (0, t_0]$, since $F(0) = f(\BF{0}) \geq 0$ and $F' \geq c\,\theta(\cdot\,\BF{1}) \geq 0$ on $(0, t_0]$ by~\eqref{eq:pL3-deriv}, the monotonicity of $\theta$ gives
\begin{equation}\label{eq:pL3-int}
F(t) \geq \int_0^t F'(u) du \geq \int_{t/2}^t F'(u) du \geq c \int_{t/2}^t \theta(u\BF{1}) du \geq \frac{ct}{2} \, \theta\Big(\frac{t}{2}\BF{1}\Big).
\end{equation}
Hence $\theta(\frac{t}{2}\BF{1}) \leq \frac{2 F(t)}{ct}$.
Now fix $n > 0$. Since $f$ is flat at the origin and $\|t\BF{1}\|_1 = dt$, we have $F(t) = o(t^{n+1})$ as $t \downarrow 0$. Therefore $\theta(\frac{t}{2}\BF{1}) = o(t^n)$, that is,
\begin{equation}\label{eq:pL3-diag}
\theta(\tau\BF{1}) = o(\tau^n) \qquad \text{as } \tau \downarrow 0 .
\end{equation}
Finally, let $\x \to \BF{0}$ and write $r = \|\x\|_1$. For $r \leq 1$, every coordinate of $\x$ is at most $r$, so $\x \leq r\BF{1} \in [0,1]^d$, and by monotonicity of $\theta$ and~\eqref{eq:pL3-diag}, $\theta(\x) \leq \theta(r\BF{1}) = o(r^n) = o(\|\x\|_1^n)$. Since $n > 0$ was arbitrary, $\theta$ is flat at the origin.

For the last claim, suppose that $f \in \C{F}_{\gamma', p, \sigma}$ for some $\gamma' \in (0,1]$, $p \geq 1$ and $\sigma \geq 0$. The first part of the proof, applied with $\theta'(\x) = \|\x\|_p^\sigma$, shows that $\theta'$ is flat at the origin. However, $\|\x\|_p \geq \|\x\|_\infty \geq \|\x\|_1 / d$, so $\theta'(\x) / \|\x\|_1^{\sigma + 1} \geq d^{-\sigma} / \|\x\|_1 \to \infty$ as $\x \to \BF{0}$, which contradicts flatness with $n = \sigma + 1$. Hence $f \notin \C{F}_{\gamma', p, \sigma}$. If $f$ were $\sigma$-OSS, then Theorem~\ref{thm:oss-equivalence} would give $f \in \C{F}_{1,1,2\sigma}$, which we have just excluded.
\end{proof}

\begin{lemma}    \label{result_beyondoss}
Let
\[
    \theta(\x)=\exp\!\left(-\frac{1}{\|\x\|_1}\right),
    \qquad \theta(\BF{0})=0.
\]
Also let $q: \mathbb{R}_{\ge 0} \to \mathbb{R}_{\ge 0}$ be a continuous non-increasing function with $q(0) > 0$, and define
\[
    f(\x)
    =
    \int_0^{\|\x\|_1}
    \exp\!\left(-\frac{1}{t}\right) q(t)\,dt.
\]
Then the following are true.
\begin{enumerate}
\item[(i)] $f$ is $1$-weakly $\theta$-up-concave.
\item[(ii)] $f$ is flat at the origin.
\item[(iii)] For any $p \geq 1$, $\sigma \geq 0$ and $\gamma \in (0, 1]$, $f$ is not $\gamma$-weakly $(p,\sigma)$-up-concave.
\item[(iv)] For any $\sigma \geq 0$, $f$ is not $\sigma$-OSS.
\item[(v)] $f$ may exhibit accumulating-then-diminishing returns. 
\end{enumerate}

\end{lemma}

\begin{table}[t]
\centering
\begin{tabular}{c|c|c}
\hline
Choice of $q(s)$ & Transition equation & Transition point $s_\star$ \\
\hline
$e^{-Ms}$ & $Ms_\star^2=1$ &
$s_\star=M^{-1/2}$ \\[2mm]

$e^{-Ms^\alpha}$, $\alpha>0$ &
$M\alpha s_\star^{\alpha+1}=1$ &
$s_\star=(M\alpha)^{-1/(\alpha+1)}$ \\[2mm]

$e^{-M(s+s^2)}$ &
$Ms_\star^2(1+2s_\star)=1$ &
unique positive root \\[2mm]

$(1+s)^{-\beta}$, $\beta>0$ &
$\displaystyle \frac{\beta s_\star^2}{1+s_\star}=1$ &
$\displaystyle s_\star=\frac{1+\sqrt{1+4\beta}}{2\beta}$ \\[3mm]

$\log(e+s)^{-\beta}$, $\beta>0$ &
$\displaystyle
\frac{\beta s_\star^2}{(e+s_\star)\log(e+s_\star)}=1$
&
unique positive root \\[3mm]

$e^{-Me^s}$ &
$Ms_\star^2e^{s_\star}=1$ &
unique positive root \\
\hline
\end{tabular}
\caption{
Examples of decay profiles $q$ (with $M > 0$) for which the marginal reward $e^{-1/s}q(s)$ first increases and then decreases. The transition point $s_\star$ solves $\frac{1}{s_\star^2} = -\frac{q'(s_\star)}{q(s_\star)}$. The profile $q$ defines $f$; an algorithm observes only values or gradients of $f$.
}
\label{tbl:example-of-q}
\end{table}

\begin{proof}
Throughout, write $s = \|\x\|_1 = \BF{1}^\top \x$ for $\x \in [0,1]^d$, and $F(s) = \int_0^s e^{-1/t} q(t) dt$, so that $f(\x) = F(s)$. The integrand $t \mapsto e^{-1/t} q(t)$, extended by $0$ at $t = 0$, is continuous on $[0, \infty)$, so $F$ is $C^1$ with $F'(s) = e^{-1/s} q(s)$.

\noindent\textbf{(i):} By the chain rule,
\begin{equation}\label{eq:pL4-grad}
\nabla f(\x) = F'(s) \BF{1} = e^{-1/s} q(s) \BF{1} ,
\end{equation}
so $f$ is $C^1$, $f \geq 0$, and $f$ is monotone because $\nabla f \geq \BF{0}$. The function $\theta$ is continuous and monotone, with $\theta(\x) > 0$ for $\x \neq \BF{0}$. For $\x \neq \BF{0}$, \eqref{eq:pL4-grad} gives $\nabla f(\x)/\theta(\x) = q(s) \BF{1}$. If $\y \geq \x \neq \BF{0}$, then $\|\y\|_1 \geq \|\x\|_1$, and since $q$ is nonincreasing, $\nabla f(\y)/\theta(\y) \leq \nabla f(\x)/\theta(\x)$ coordinate-wise. Theorem~\ref{thm:theta-normalized-DR} with $\gamma = 1$ gives (i).

\noindent\textbf{(ii):} Fix $n > 0$ and choose $N > n$. Since $e^{-1/t}/t^N \to 0$ as $t \downarrow 0$, there is $\delta > 0$ with $e^{-1/t} \leq t^N$ for all $t \in (0, \delta]$. Since $q$ is nonincreasing, $q(t) \leq q(0)$. Hence, for $s \leq \delta$,
\begin{equation}\label{eq:pL4-flat}
f(\x) = \int_0^s e^{-1/t} q(t) dt \leq q(0) \int_0^s t^N dt = \frac{q(0)}{N + 1} s^{N + 1},
\end{equation}
so $f(\x)/\|\x\|_1^n \leq \frac{q(0)}{N+1} s^{N + 1 - n} \to 0$ as $\x \to \BF{0}$, because $N + 1 - n > 0$.

\noindent\textbf{(iii) and (iv):} Since $q$ is continuous and $q(0) > 0$, there is $\eta > 0$ with $q > 0$ on $[0, \eta)$, so $f(\x) > 0$ whenever $0 < \|\x\|_1 < \eta$, and $f$ is not identically zero. By (i) and (ii), Lemma~\ref{lem:flat-functions-are-not-OSS} applies and gives (iii) and (iv).

\noindent\textbf{(v):} Let $q(s) = e^{-\psi(s)}$, where $\psi : [0, \infty) \to \B{R}$ is continuous and nondecreasing, differentiable on $(0, \infty)$, and such that $s \mapsto s^2 \psi'(s)$ is strictly increasing on $(0, \infty)$ with range containing $1$; then $q$ is continuous, nonincreasing and $q(0) = e^{-\psi(0)} > 0$. The calculation below only differentiates at $s > 0$. We have $F'(s) = e^{-1/s - \psi(s)}$, and for $s > 0$,
\begin{equation}\label{eq:pL4-second}
F''(s) = \left( \frac{1}{s^2} - \psi'(s) \right) e^{-1/s - \psi(s)} = \frac{1 - s^2 \psi'(s)}{s^2} \, e^{-1/s - \psi(s)} .
\end{equation}
Let $s_\star$ be the unique solution of $s_\star^2 \psi'(s_\star) = 1$. Since $s \mapsto s^2 \psi'(s)$ is strictly increasing, \eqref{eq:pL4-second} shows that $F'' > 0$ on $(0, s_\star)$ and $F'' < 0$ on $(s_\star, \infty)$. Thus the marginal reward $F'$ increases on $(0, s_\star)$ and decreases on $(s_\star, \infty)$. On $[0,1]^d$, where $s \in [0, d]$, both phases occur when $s_\star < d$. Table~\ref{tbl:example-of-q} lists examples of $q$ of this form.
\end{proof}

\begin{lemma}\label{lem:activation}
Let $\BF{a} \in \B{R}^d_{>0}$, let $M > 0$, and let $f$ and $\theta$ be as in Example~\ref{ex:activation}.
Then the following are true.
\begin{enumerate}
\item[(i)] $f$ is $1$-weakly $\theta$-up-concave.
\item[(ii)] $f$ is flat at the origin.
\item[(iii)] For any $p \geq 1$, $\sigma \geq 0$ and $\gamma \in (0, 1]$, $f$ is not $\gamma$-weakly $(p,\sigma)$-up-concave, and for any $\sigma \geq 0$, $f$ is not $\sigma$-OSS.
\item[(iv)] As a function of $s = \BF{a}^\top \x$, the marginal gain of $f$ is increasing for $s < 1/\sqrt{M}$ and decreasing for $s > 1/\sqrt{M}$; both phases occur on $[0,1]^d$ when $M > \|\BF{a}\|_1^{-2}$.
\end{enumerate}

\end{lemma}

\begin{proof}
Throughout, write $s = \BF{a}^\top \x$ and $F(s) = \int_0^s \exp(-1/t - Mt) dt$, so that $f(\x) = F(s)$. The integrand, extended by $0$ at $t = 0$, is continuous on $[0,\infty)$, so $F$ is $C^1$ with $F'(s) = \exp(-1/s - Ms)$.

\noindent\textbf{(i):} By the chain rule,
\begin{equation}\label{eq:pL5-grad}
\nabla f(\x) = F'(s) \BF{a} = \exp\Big(-\frac{1}{s} - Ms\Big) \BF{a},
\end{equation}
so $f$ is $C^1$ on $[0,1]^d$, $f \geq 0$, and $f$ is monotone since $\nabla f \geq \BF{0}$. The function $\theta$ is continuous and monotone, and $\theta(\x) > 0$ for $\x \neq \BF{0}$, because $\BF{a} \in \B{R}^d_{>0}$ gives $s > 0$. For $\x \neq \BF{0}$, \eqref{eq:pL5-grad} gives $\nabla f(\x)/\theta(\x) = e^{-Ms}\BF{a}$. If $\y \geq \x \neq \BF{0}$, then $\BF{a}^\top \y \geq \BF{a}^\top \x$, so $e^{-M\BF{a}^\top\y} \leq e^{-M\BF{a}^\top\x}$ and $\nabla f(\y)/\theta(\y) \leq \nabla f(\x)/\theta(\x)$ coordinate-wise. Theorem~\ref{thm:theta-normalized-DR} with $\gamma = 1$ gives (i).

\noindent\textbf{(ii):} Let $A = \|\BF{a}\|_\infty$, so that $0 \leq s \leq A\|\x\|_1$. For $0 < t \leq s$, we have $\exp(-1/t - Mt) \leq e^{-1/t} \leq e^{-1/s}$, hence
\begin{equation}\label{eq:pL5-flat}
f(\x) \leq s e^{-1/s} \leq A\|\x\|_1 \exp\Big(-\frac{1}{A\|\x\|_1}\Big),
\end{equation}
where the second inequality uses that $u \mapsto u e^{-1/u}$ is increasing on $(0,\infty)$. For every $n > 0$, the right-hand side of~\eqref{eq:pL5-flat} divided by $\|\x\|_1^n$ tends to $0$ as $\x \to \BF{0}$, because $e^{-1/u}$ tends to $0$ faster than any power of $u$ as $u \downarrow 0$.

\noindent\textbf{(iii):} Since $F'(s) > 0$ for $s > 0$, the function $f$ is not identically zero. By (i) and (ii), Lemma~\ref{lem:flat-functions-are-not-OSS} gives (iii).

\noindent\textbf{(iv):} For $s > 0$, $F''(s) = (1/s^2 - M) F'(s)$, and $F'(s) > 0$. Hence $F'$ is increasing for $s < 1/\sqrt{M}$ and decreasing for $s > 1/\sqrt{M}$. On $[0,1]^d$, $s$ ranges over $[0, \|\BF{a}\|_1]$, so both phases occur when $1/\sqrt{M} < \|\BF{a}\|_1$, that is, when $M > \|\BF{a}\|_1^{-2}$.
\end{proof}

\begin{proposition}\label{prop:multivariate}
Let $c > 0$, let $\theta(\x) = \exp(-c/\|\x\|_1)$ for $\x \neq \BF{0}$ and $\theta(\BF{0}) = 0$, and let $g : [0,1]^d \to \B{R}_{>0}$ be a $C^2$ function that is coordinate-wise nonincreasing and DR-submodular, i.e., $\partial_{ij}^2 g \leq 0$ for all $i, j$.
Define $F = \theta g$, with $F(\BF{0}) = 0$.
Then $F$ is $C^1$, and $F$ is monotone if and only if
\begin{equation}\label{eq:log-decay}
-\partial_i \log g(\x) \leq \frac{c}{\|\x\|_1^2}
\quad \text{for all } i \text{ and all } \x \neq \BF{0}.
\end{equation}
If~\eqref{eq:log-decay} holds, then $F$ is $1$-weakly $\theta$-up-concave.

\end{proposition}

\begin{proof}
For $\x \neq \BF{0}$, write $s = \|\x\|_1 = \BF{1}^\top \x > 0$. Since $\partial_i s = 1$ and $\theta(\x) = e^{-c/s}$, the chain rule gives $\partial_i \theta(\x) = \frac{c}{s^2} \theta(\x)$, and the product rule gives
\begin{equation}\label{eq:pP2-grad}
\partial_i F(\x) = \theta(\x) \partial_i g(\x) + g(\x) \partial_i \theta(\x) = \theta(\x) q_i(\x),
\qquad\text{where } q_i(\x) := \partial_i g(\x) + \frac{c}{s^2} g(\x).
\end{equation}

\emph{Regularity.} Since $g$ is $C^2$ on the compact set $[0,1]^d$, both $g$ and $\nabla g$ are bounded there. Since $e^{-c/s}/s^2 \to 0$ as $s \downarrow 0$, \eqref{eq:pP2-grad} gives $\nabla F(\x) \to \BF{0}$ as $\x \to \BF{0}$. Moreover, $0 \leq F(\x) \leq \theta(\x) \max g$ and $\theta(\x)/\|\x\|_1 \to 0$, so $F$ is differentiable at $\BF{0}$ with $\nabla F(\BF{0}) = \BF{0}$. Hence $F$ is $C^1$ on $[0,1]^d$, and $F \geq 0$.

\emph{Monotonicity.} A $C^1$ function on $[0,1]^d$ is monotone if and only if all its partial derivatives are nonnegative. Since $\nabla F(\BF{0}) = \BF{0}$ and $\theta(\x) > 0$ for $\x \neq \BF{0}$, \eqref{eq:pP2-grad} shows that $F$ is monotone if and only if $q_i(\x) \geq 0$ for all $i$ and all $\x \neq \BF{0}$. Dividing by $g(\x) > 0$, the inequality $q_i(\x) \geq 0$ is equivalent to $-\partial_i g(\x)/g(\x) \leq c/s^2$, which is~\eqref{eq:log-decay}.

\emph{Up-concavity.} Assume~\eqref{eq:log-decay}, so that $F$ is monotone. For $\x \neq \BF{0}$, differentiating $q_i$ and using $\partial_j s = 1$ gives
\begin{equation}\label{eq:pP2-dq}
\partial_j q_i(\x) = \partial_{ij}^2 g(\x) + \frac{c}{s^2} \partial_j g(\x) - \frac{2c}{s^3} g(\x) .
\end{equation}
Each term in~\eqref{eq:pP2-dq} is nonpositive: $\partial_{ij}^2 g \leq 0$ by DR-submodularity, $\partial_j g \leq 0$ because $g$ is nonincreasing, and $g > 0$. So $\nabla q_i(\x) \leq \BF{0}$ for all $i$ and all $\x \neq \BF{0}$.
Now let $\y \geq \x \neq \BF{0}$. Every point $\x + t(\y - \x)$ with $t \in [0,1]$ is at least $\x$ and hence nonzero, so $q_i$ is $C^1$ along this segment, and
\[
q_i(\y) - q_i(\x) = \int_0^1 \bra \nabla q_i(\x + t(\y - \x)), \y - \x \ket dt \leq 0,
\]
because $\nabla q_i \leq \BF{0}$ and $\y - \x \geq \BF{0}$. By~\eqref{eq:pP2-grad}, $q_i = \partial_i F/\theta$, so $\nabla F(\y)/\theta(\y) \leq \nabla F(\x)/\theta(\x)$ coordinate-wise. Since $\theta$ is continuous and monotone with $\theta(\x) > 0$ for $\x \neq \BF{0}$, Theorem~\ref{thm:theta-normalized-DR} with $\gamma = 1$ shows that $F$ is $1$-weakly $\theta$-up-concave.
\end{proof}

\begin{remark}
Condition~\eqref{eq:log-decay} holds whenever $g$ is large compared to its partial derivatives.
Indeed, if $|\partial_i g| \leq L$ on $[0,1]^d$ for some $L > 0$ and all $i$, and $g \geq d^2 L / c$, then, since $\|\x\|_1 \leq d$,
\[
-\partial_i \log g(\x) = \frac{-\partial_i g(\x)}{g(\x)} \leq \frac{L}{d^2 L / c} = \frac{c}{d^2} \leq \frac{c}{\|\x\|_1^2}.
\]
Thus, if $h : [0,1]^d \to \B{R}_{\geq 0}$ is $C^2$, coordinate-wise nonincreasing and DR-submodular with $|\partial_i h| \leq L$ for some $L > 0$ and all $i$, then $g = h + d^2 L / c$ satisfies the assumptions of Proposition~\ref{prop:multivariate}, and $F = \theta g$ is $1$-weakly $\theta$-up-concave.
For example, $h(\x) = d - \frac{1}{2} \|\x\|_2^2$ satisfies these conditions with $L = 1$, so
\[
F(\x) = \exp\Big(-\frac{c}{\|\x\|_1}\Big) \Big( d + \frac{d^2}{c} - \frac{1}{2}\|\x\|_2^2 \Big)
\]
is $1$-weakly $\theta$-up-concave.
Here $\partial_i F(\x)/\theta(\x) = \frac{c}{\|\x\|_1^2} g(\x) - x_i$, so for $d \geq 2$ the direction of $\nabla F(\x)$ changes with $\x$.
In particular, $F$ is not of the form $\phi(\BF{a}^\top \x)$ for any $\BF{a} \in \B{R}^d$ and differentiable $\phi : \B{R} \to \B{R}$, and its level sets are not parallel hyperplanes.
Since $F$ is flat at the origin and not identically zero, Lemma~\ref{lem:flat-functions-are-not-OSS} shows that it is neither $\sigma$-OSS nor $\gamma$-weakly $(p,\sigma)$-up-concave for any choice of parameters.
Functions in classes with different scaling functions can also be combined. If $f_j \in \C{F}_{\gamma_j, \theta_j}$ for $j = 1, \dots, m$, let $\gamma = \min_j \gamma_j$ and $\Theta = \prod_j \theta_j$. For $\y \geq \x \neq \BF{0}$, each ratio $\theta_j(\y)/\theta_j(\x)$ is at least $1$, so $\Theta(\y)/\Theta(\x) \geq \theta_j(\y)/\theta_j(\x)$ for every $j$. Since the gradient inner products in~\eqref{eq:weak-theta} are nonnegative, the inequalities for $(\gamma_j, \theta_j)$ imply those for $(\gamma, \Theta)$, so each $f_j$ belongs to $\C{F}_{\gamma, \Theta}$. By Lemma~\ref{lem:function-class-convex-cone}, every nonnegative combination $\sum_j a_j f_j$ also belongs to $\C{F}_{\gamma, \Theta}$.

\end{remark}

\section{\texorpdfstring{Proof of Corollary~\ref{cor:main-alpha}}{Proof of Corollary 2}}
\label{apx:kstar-away}

\begin{proof}
Under the assumptions of Theorem~\ref{thm:linearizable}, $\C{K} \subseteq [0,1]^d$ is closed and convex and contains a point other than $\BF{0}$, and $\theta$ is continuous and monotone with $\theta(\x) > 0$ for $\x \neq \BF{0}$.

\emph{A point of $\C{K}$ with maximal support.}
Let $I$ be the set of coordinates $i$ such that $x_i > 0$ for some $\x \in \C{K}$. Since $\C{K} \subseteq [0,1]^d$ contains a point other than $\BF{0}$, the set $I$ is nonempty. For each $i \in I$, choose $\x^{(i)} \in \C{K}$ with $x^{(i)}_i > 0$, and let
\begin{equation}\label{eq:pK-q}
\BF{q} = \frac{1}{|I|} \sum_{i \in I} \x^{(i)} .
\end{equation}
Since $\C{K}$ is convex, $\BF{q} \in \C{K}$. For $i \in I$, every term in~\eqref{eq:pK-q} is nonnegative and the term $\x^{(i)}$ has a positive $i$-th coordinate, so $q_i > 0$. For $i \notin I$, every point of $\C{K}$ has $i$-th coordinate $0$, so $q_i = 0$. Let $\delta = \min_{i \in I} q_i > 0$.

\emph{Every maximal point has $\ell_1$ norm at least $\delta$.}
Let $\vv \in \C{K}^{\op m}$. Suppose, for contradiction, that $v_i < q_i$ for all $i \in I$. For $i \notin I$, we have $v_i = 0 = q_i$ because $\vv \in \C{K}$. Hence $\BF{q} \geq \vv$ and $\BF{q} \neq \vv$ (since $I \neq \emptyset$), and $\BF{q} \in \C{K}$, which contradicts the maximality of $\vv$. Therefore $v_i \geq q_i \geq \delta$ for some $i \in I$, and since $\vv \geq \BF{0}$,
\begin{equation}\label{eq:pK-norm}
\BF{1}^\top \vv \geq v_i \geq \delta \qquad \text{for all } \vv \in \C{K}^{\op m}.
\end{equation}

\emph{$\BF{0} \notin \C{K}^*$.}
By~\eqref{eq:pK-norm}, $\C{K}^{\op m}$ is contained in the halfspace $H = \{\x : \BF{1}^\top \x \geq \delta\}$. Since $H$ is closed and convex, it also contains $\C{K}^* = \overline{\op{conv}}(\C{K}^{\op m})$. Since $\BF{0} \notin H$, we get $\BF{0} \notin \C{K}^*$.

\emph{The infimum is positive.}
Define $\Phi(\x) = \int_0^1 \theta(s\x) ds$ for $\x \in [0,1]^d$. If $\x \neq \BF{0}$, then $s\x \neq \BF{0}$ for $s \in (0,1]$, so the continuous function $s \mapsto \theta(s\x)$ is positive on $(0,1]$ and $\Phi(\x) > 0$. Since $\theta$ is continuous on the compact set $[0,1]^d$, it is uniformly continuous there, and hence $\Phi$ is continuous. The set $\C{K}^*$ is closed (by definition) and contained in $[0,1]^d$ (since $\C{K}^{\op m} \subseteq \C{K}$ and $\C{K}$ is closed and convex), so it is compact. Therefore $\Phi$ attains its minimum over $\C{K}^*$ at some point $\x^\star \in \C{K}^*$. Since $\x^\star \neq \BF{0}$, we get $\inf_{\x \in \C{K}^*} \Phi(\x) = \Phi(\x^\star) > 0$, and therefore $\alpha = 1 - \exp(-\gamma \Phi(\x^\star)/R_\theta) > 0$.

\emph{The inequality.}
Since $\C{K}$ is closed and convex and $\C{K}^{\op m} \subseteq \C{K}$, we have $\C{K}^* \subseteq \C{K}$. Hence Theorem~\ref{thm:linearizable} applies to $\x, \y \in \C{K}^*$ and gives
\[
\bra \F{g}(f, \x), \y - \x \ket \geq \alpha_\x f(\y) - f(\x).
\]
For $\x \in \C{K}^*$, we have $\Phi(\x) \geq \inf_{\z \in \C{K}^*} \Phi(\z)$, and $t \mapsto 1 - \exp(-\gamma t / R_\theta)$ is nondecreasing, so $\alpha_\x \geq \alpha$. Since $f(\y) \geq 0$, it follows that $\alpha_\x f(\y) \geq \alpha f(\y)$, which gives the claimed inequality.
\end{proof}

\section{Proof of Theorem~\ref{thm:regret_main}}
\label{apx:regret_main}

\begin{proof}
Recall that $\C{G}(\C{Q}_f, \x) = e^{l(Z, \x)} \C{Q}_f(Z\x)$, where $Z$ is drawn uniformly from $[0,1]$ using fresh randomization. Throughout, we use only the oracle assumption of Section~\ref{sec:background}: the oracle output at a query point $\y$ has conditional mean $\nabla f(\y)$, given the history of the game and the query point.

\emph{$\C{G}$ is an unbiased query algorithm for $\F{g}$.}
Fix $\x \in \C{K}$. Since $\C{K}$ is convex and contains $\BF{0}$ and $\x$, the query point $Z\x$ lies in $\C{K}$. The query point is determined by $Z$, so conditional unbiasedness gives $\B{E}[\C{Q}_f(Z\x) \mid Z] = \nabla f(Z\x)$. By the tower rule of conditional expectation,
\begin{equation}\label{eq:pG-unbiased}
\begin{aligned}
\B{E}\left[ \C{G}(\C{Q}_f, \x) \right]
&= \B{E}\Big[ e^{l(Z, \x)} \, \B{E}[\C{Q}_f(Z\x) \mid Z] \Big]
= \B{E}\left[ e^{l(Z, \x)} \nabla f(Z\x) \right] \\
&= \int_0^1 e^{l(r, \x)} \nabla f(r\x) dr
= \F{g}(f, \x),
\end{aligned}
\end{equation}
where the third equality holds because $Z$ has density $1$ on $[0,1]$.

\emph{Boundedness.}
By~\eqref{eq:pF-l}, $l(Z, \x) \leq 0$, so $0 < e^{l(Z, \x)} \leq 1$. If $\|\C{Q}_f(\z)\| \leq B_1$ for all $\z \in \C{K}$, then, since $Z\x \in \C{K}$,
\begin{equation}\label{eq:pG-bound}
\|\C{G}(\C{Q}_f, \x)\| = e^{l(Z, \x)} \|\C{Q}_f(Z\x)\| \leq B_1 .
\end{equation}

\emph{Setting up the regret bound.}
Fix an adversary in $\op{Adv}^{\op{o}}_1(\C{F}, B_1)$, an interval $[a, b]$ and a compact comparator set $\C{U} \subseteq (\C{K}^*)^T$. Since the adversary is oblivious, the functions $f_1, \dots, f_T$ are fixed before the game. Each $f_t$ is continuous and $\C{U}$ is compact, so we may fix, before the game, a deterministic maximizer
\begin{equation}\label{eq:pG-ustar}
\uu^\star = (\uu^\star_1, \dots, \uu^\star_T) \in \argmax_{\uu \in \C{U}} \sum_{t=a}^b f_t(\uu_t).
\end{equation}
Let $\x_t \in \C{K}^*$ be the action played at round $t$, let $\BF{o}_t$ be the oracle output at the query point $Z_t \x_t$, and let $\BF{g}_t = e^{l(Z_t, \x_t)} \BF{o}_t$ be the vector passed to $\mathcal{A}(\C{K}^*)$. Let $\C{H}_t$ denote the history before the round-$t$ sampling step. Since $\mathcal{A}(\C{K}^*)$ is deterministic, $\x_t$ is $\C{H}_t$-measurable, and $Z_t$ is uniform on $[0,1]$ and independent of $\C{H}_t$, because it is drawn with fresh randomization. Conditional unbiasedness with respect to the complete pre-query history gives
\begin{equation}\label{eq:pG-oracle}
\B{E}[\BF{o}_t \mid \C{H}_t, Z_t] = \nabla f_t(Z_t \x_t).
\end{equation}
Therefore, by the tower rule,
\begin{equation}\label{eq:pG-cond}
\begin{aligned}
\B{E}[\BF{g}_t \mid \C{H}_t]
&= \B{E}\Big[ e^{l(Z_t, \x_t)} \, \B{E}[\BF{o}_t \mid \C{H}_t, Z_t] \,\Big|\, \C{H}_t \Big]
= \B{E}\Big[ e^{l(Z_t, \x_t)} \nabla f_t(Z_t \x_t) \,\Big|\, \C{H}_t \Big] \\
&= \int_0^1 e^{l(r, \x_t)} \nabla f_t(r \x_t) dr
= \F{g}(f_t, \x_t),
\end{aligned}
\end{equation}
where the third equality uses that $\x_t$ is $\C{H}_t$-measurable and that, given $\C{H}_t$, $Z_t$ is still uniform on $[0,1]$.

\emph{From the nonlinear to the linear regret.}
Since $\x_t, \uu^\star_t \in \C{K}^*$, Corollary~\ref{cor:main-alpha} gives, for every $t$,
\begin{equation}\label{eq:pG-round}
\alpha f_t(\uu^\star_t) - f_t(\x_t) \leq \bra \F{g}(f_t, \x_t), \uu^\star_t - \x_t \ket .
\end{equation}
The vector $\uu^\star_t - \x_t$ is $\C{H}_t$-measurable, so by~\eqref{eq:pG-cond} and the tower rule,
\begin{equation}\label{eq:pG-tower}
\B{E} \bra \F{g}(f_t, \x_t), \uu^\star_t - \x_t \ket
= \B{E} \Big[ \bra \B{E}[\BF{g}_t \mid \C{H}_t], \uu^\star_t - \x_t \ket \Big]
= \B{E} \bra \BF{g}_t, \uu^\star_t - \x_t \ket .
\end{equation}
Taking expectations in~\eqref{eq:pG-round}, using~\eqref{eq:pG-tower}, summing over $t \in [a, b]$, and using~\eqref{eq:pG-ustar}, we obtain
\begin{equation}\label{eq:pG-sum}
\B{E} \left[ \alpha \max_{\uu \in \C{U}} \sum_{t=a}^b f_t(\uu_t) - \sum_{t=a}^b f_t(\x_t) \right]
\leq \B{E} \sum_{t=a}^b \bra \BF{g}_t, \uu^\star_t - \x_t \ket
\leq \B{E} \left[ \max_{\vv \in \C{U}} \sum_{t=a}^b \bra \BF{g}_t, \vv_t \ket - \sum_{t=a}^b \bra \BF{g}_t, \x_t \ket \right].
\end{equation}

\emph{Bounding the linear regret.}
Since $\mathcal{A}(\C{K}^*)$ is deterministic and receives only $\BF{g}_1, \dots, \BF{g}_{t-1}$ before choosing $\x_t$, for every realization of the sequence $(\BF{g}_t)_{t=1}^T$ the actions $\x_t$ are exactly those that $\mathcal{A}(\C{K}^*)$ plays against the fixed sequence of linear rewards $\vv \mapsto \bra \BF{g}_t, \vv \ket$. By~\eqref{eq:pG-bound}, these rewards belong to $\C{L}(\C{K}^*)[B_1]$. Hence the worst-case linear regret bound of $\mathcal{A}(\C{K}^*)$ applies pathwise, and the quantity inside the expectation on the right-hand side of~\eqref{eq:pG-sum} is at most $\mathcal{R}^{\mathcal{A}(\C{K}^*)}_{1, \op{Adv}_1(\C{L}(\C{K}^*)[B_1])}(\C{U})[a,b]$ for every realization. Taking expectations, and then the supremum over oblivious adversaries on the left-hand side, proves the claimed inequality.

\emph{Comparators in $\C{K}$.}
For static and adaptive regret, the comparator on $[a,b]$ is a single point $\uu$ used in all rounds $t \in [a,b]$. Let $\uu \in \C{K}$. The set $\{\BF{w} \in \C{K} : \BF{w} \geq \uu\}$ is compact and nonempty, so $\|\cdot\|_1$ attains its maximum on it at some point $\vv$. This $\vv$ lies in $\C{K}^{\op m}$: if $\z \in \C{K}$ satisfied $\z \geq \vv$ and $\z \neq \vv$, then $\z \geq \uu$ and $\|\z\|_1 > \|\vv\|_1$, a contradiction. Thus $\vv \in \C{K}^{\op m} \subseteq \C{K}^*$ and $\vv \geq \uu$, and since every $f_t$ is monotone, $\sum_{t=a}^b f_t(\vv) \geq \sum_{t=a}^b f_t(\uu)$. Hence $\max_{\uu \in \C{K}} \sum_{t=a}^b f_t(\uu) = \max_{\uu \in \C{K}^*} \sum_{t=a}^b f_t(\uu)$, and the left-hand side may be maximized over $\C{K}$.
\end{proof}

\section{\texorpdfstring{Proof of Proposition~\ref{prop:outer-approx}}{Proof of the proposition on outer approximations}}\label{apx:outer-approx}

\begin{proof}
Throughout, write $R = R_\theta$ and $\Phi(\x) = \int_0^1 \theta(s\x) ds$, so that
\begin{equation}\label{eq:pH-def}
\alpha_\x = 1 - \exp\left( -\frac{\gamma}{R} \Phi(\x) \right),
\qquad
\alpha(\C{K}') = 1 - \exp\left( -\frac{\gamma}{R} \inf_{\x \in \C{K}'} \Phi(\x) \right).
\end{equation}
The map $t \mapsto 1 - \exp(-\gamma t / R)$ is nondecreasing on $\B{R}$.

\emph{The inequality of Corollary~\ref{cor:main-alpha} on $\C{K}'$.}
Fix $\x, \y \in \C{K}'$. Since $\C{K}' \subseteq \C{K}$, Theorem~\ref{thm:linearizable} gives
\begin{equation}\label{eq:pH-thm}
\bra \F{g}(f, \x), \y - \x \ket \geq \alpha_\x f(\y) - f(\x).
\end{equation}
Since $\x \in \C{K}'$, we have $\Phi(\x) \geq \inf_{\z \in \C{K}'} \Phi(\z)$, and by~\eqref{eq:pH-def} and the monotonicity of $t \mapsto 1 - \exp(-\gamma t/R)$, $\alpha_\x \geq \alpha(\C{K}')$. Since $f(\y) \geq 0$, it follows that $\alpha_\x f(\y) \geq \alpha(\C{K}') f(\y)$, and~\eqref{eq:pH-thm} gives
\begin{equation}\label{eq:pH-cor}
\bra \F{g}(f, \x), \y - \x \ket \geq \alpha(\C{K}') f(\y) - f(\x)
\qquad \text{for all } \x, \y \in \C{K}'.
\end{equation}

\emph{The regret bound on $\C{K}'$.}
The proof of Theorem~\ref{thm:regret_main} in Appendix~\ref{apx:regret_main} uses three facts about $\C{K}^*$: the inequality of Corollary~\ref{cor:main-alpha}, that $\C{K}^* \subseteq \C{K}$ (so that the query points $Z_t \x_t$ lie in $\C{K}$), and that every point of $\C{K}$ is dominated by a point of $\C{K}^*$. The set $\C{K}'$ has the same three properties: the first is~\eqref{eq:pH-cor}, the second holds by assumption, and the third holds because $\C{K}^{\op m} \subseteq \C{K}^* \subseteq \C{K}'$. Hence that proof applies verbatim with $\C{K}'$ and $\alpha(\C{K}')$ in place of $\C{K}^*$ and $\alpha$. If $\alpha(\C{K}') = 0$, the resulting bound compares $\sum_t f_t(\x_t)$ with $0$, so it gives no approximation guarantee.

\emph{Comparison with $\C{K}^*$.}
Since $\C{K}^* \subseteq \C{K}'$, an infimum over $\C{K}'$ is at most the infimum over $\C{K}^*$. By~\eqref{eq:pH-def} and monotonicity, $\alpha(\C{K}') \leq \alpha(\C{K}^*)$.

\emph{The case $\theta(\x) = \|\x\|_p^\sigma$.}
Suppose first that $\sigma > 0$. Since $\|s\x\|_p = s\|\x\|_p$ for $s \geq 0$,
\begin{equation}\label{eq:pH-phi}
\Phi(\x) = \|\x\|_p^\sigma \int_0^1 s^\sigma ds = \frac{\|\x\|_p^\sigma}{\sigma + 1}.
\end{equation}
Since $t \mapsto t^\sigma$ is nondecreasing on $[0,\infty)$, we have $R_\theta(\C{K}) = \max_{\vv, \uu \in \C{K}} \|\vv \vee \uu\|_p^\sigma = R_p(\C{K})^\sigma$, and $\inf_{\x \in \C{K}'} \|\x\|_p^\sigma = (\inf_{\x \in \C{K}'} \|\x\|_p)^\sigma$. The set $\C{K}'$ is closed and contained in $[0,1]^d$, hence compact, and $\|\cdot\|_p$ is continuous, so the infimum is a minimum. Substituting into~\eqref{eq:pH-def} gives
\[
\alpha(\C{K}') = 1 - \exp\left( -\frac{\gamma}{\sigma + 1} \left( \frac{\min_{\x \in \C{K}'} \|\x\|_p}{R_p(\C{K})} \right)^\sigma \right).
\]
If $\sigma = 0$, then $\theta \equiv 1$, so $\Phi \equiv 1$ and $R = 1$, and~\eqref{eq:pH-def} gives $\alpha(\C{K}') = 1 - e^{-\gamma}$.

\emph{The Lipschitz case.}
Assume that $\theta$ is $L$-Lipschitz with respect to $\|\cdot\|$, and let
\begin{equation}\label{eq:pH-Istar}
I^* := \inf_{\vv \in \C{K}^*} \Phi(\vv),
\qquad\text{so that}\qquad
1 - \alpha(\C{K}^*) = \exp\left( -\frac{\gamma I^*}{R} \right).
\end{equation}
Fix $\x \in \C{K}'$ and $\eta > 0$. Since $\x$ is within distance $\varepsilon$ of $\C{K}^*$, there is $\vv \in \C{K}^*$ with $\|\x - \vv\| \leq \varepsilon + \eta$. For every $s \in [0,1]$, the Lipschitz property gives
\begin{equation}\label{eq:pH-lip}
\theta(s\x) \geq \theta(s\vv) - L \|s\x - s\vv\| = \theta(s\vv) - L s \|\x - \vv\| \geq \theta(s\vv) - L s (\varepsilon + \eta).
\end{equation}
Integrating~\eqref{eq:pH-lip} over $s \in [0,1]$, and using $\int_0^1 s \, ds = \frac{1}{2}$ and $\Phi(\vv) \geq I^*$,
\begin{equation}\label{eq:pH-int}
\Phi(\x) \geq \Phi(\vv) - \frac{L(\varepsilon + \eta)}{2} \geq I^* - \frac{L(\varepsilon + \eta)}{2}.
\end{equation}
The left-hand side of~\eqref{eq:pH-int} does not depend on $\eta$, so letting $\eta \downarrow 0$ gives $\Phi(\x) \geq I^* - L\varepsilon/2$. Also $\Phi(\x) \geq 0$ because $\theta \geq 0$. Taking the infimum over $\x \in \C{K}'$,
\begin{equation}\label{eq:pH-inf}
\inf_{\x \in \C{K}'} \Phi(\x) \geq \max\left\{ 0, \, I^* - \frac{L\varepsilon}{2} \right\}.
\end{equation}
By~\eqref{eq:pH-def} and monotonicity, \eqref{eq:pH-inf} implies $\alpha(\C{K}') \geq 1 - \exp\big( -\frac{\gamma}{R} \max\{0, I^* - \frac{L\varepsilon}{2}\} \big)$. We now compare this with the claimed bound.
If $I^* \geq L\varepsilon/2$, then by~\eqref{eq:pH-Istar},
\[
1 - \exp\left( -\frac{\gamma}{R} \Big( I^* - \frac{L\varepsilon}{2} \Big) \right)
= 1 - \exp\left( -\frac{\gamma I^*}{R} \right) \exp\left( \frac{\gamma L \varepsilon}{2R} \right)
= 1 - (1 - \alpha(\C{K}^*)) \exp\left( \frac{\gamma L \varepsilon}{2R} \right),
\]
and this quantity is nonnegative, so it equals the maximum in the claim.
If $I^* < L\varepsilon/2$, the lower bound above is $1 - e^0 = 0$, while $1 - (1 - \alpha(\C{K}^*)) \exp(\frac{\gamma L\varepsilon}{2R}) = 1 - \exp(\frac{\gamma}{R}(\frac{L\varepsilon}{2} - I^*)) < 0$, so the maximum in the claim is $0$.
In both cases, $\alpha(\C{K}') \geq \max\{0, 1 - (1 - \alpha(\C{K}^*)) \exp(\frac{\gamma L \varepsilon}{2R})\}$. The upper bound $\alpha(\C{K}') \leq \alpha(\C{K}^*)$ was shown above.
\end{proof}

\section{Separation Oracle for \texorpdfstring{$\C{K}^*$}{K-star} via Benson-Type Methods}\label{apx:benson}

Let $\C{K} \subseteq [0,1]^d$ be a nonempty rational polyhedron.
Recall that
\[
\C{K}^{\op m} = \{\x \in \C{K} : \nexists\, \y \in \C{K},\ \y \ge \x,\ \y \neq \x\}, 
\qquad
\C{K}^* = \overline{\operatorname{conv}}(\C{K}^{\op m}).
\]
We assume that $\C{K}$ is given by a finite rational inequality description, and we use exact linear-programming and polyhedral-enumeration subroutines.

\paragraph{Vector optimization formulation.}
Consider
\[
(\mathsf P)\qquad \min_{\le \mathbb{R}_+^d} (-I_d)\x \quad \text{s.t. } \x \in \C{K}.
\]
Then $\x \in \C{K}$ is a minimizer if and only if $\x \in \C{K}^{\op m}$.

For a linear map $P$ and set $S$, write $P[S] = \{ P \x \mid \x \in S \}$.
We use the following standard result about the classical primal Benson algorithm (see \cite{benson98_outer_approx_algor_gener_all,lohne11_vector_optim_infim_suprem}).

\begin{theorem}[Primal Benson, bounded case]
Let $\C{K}$ be a nonempty bounded polyhedron. For
\[
\min_{\le C} P \x \quad \text{s.t. } \x \in \C{K},
\quad C=\mathbb{R}_+^d,
\]
the primal Benson algorithm terminates finitely and returns a finite set $\C{E} \subseteq \C{K}$ of minimizers such that
\[
\operatorname{conv}(P[\C{E}]) + C = P[\C{K}] + C.
\]
\end{theorem}

Applying with $P=-I_d$, we obtain $\C{E} \subseteq \C{K}^{\op m}$ with
\[
\C{K}^* = \operatorname{conv}(\C{E}).
\]
Indeed, let $\x \in \C{K}^{\op m}$. The identity above with $P = -I_d$ gives $-\x = -\vv + \BF{c}$ for some $\vv \in \operatorname{conv}(\C{E}) \subseteq \C{K}$ and $\BF{c} \geq \BF{0}$. Thus $\vv \geq \x$, and maximality of $\x$ forces $\vv = \x$. Hence $\C{K}^{\op m} \subseteq \operatorname{conv}(\C{E})$, and since $\operatorname{conv}(\C{E})$ is closed and convex, $\C{K}^* \subseteq \operatorname{conv}(\C{E})$. Conversely, $\C{E} \subseteq \C{K}^{\op m}$ gives $\operatorname{conv}(\C{E}) \subseteq \C{K}^*$.

\paragraph{Separation oracle (primal version).}
Membership in $\C{K}^*$ reduces to convex hull feasibility. If it is infeasible, Farkas' lemma gives a separating hyperplane.

The numbers of generated vertices and faces can be exponential in the description size of $\C{K}$, so this construction is finite but potentially expensive, and it does not give a polynomial-time oracle for $\C{K}^*$.

Remark~\ref{rem:solvable} and Proposition~\ref{prop:outer-approx} describe the alternative of running the algorithm on a larger set $\C{K}'$.

\section{Proof of Theorem~\ref{thm:matroid-radial}}\label{apx:matroid-radial}

\begin{proof}
Write $\rho = \rho(\C M)$ and, for $S \subseteq V$, let $\BF{1}_S$ be its indicator vector. We use repeatedly that, for $\x, \y \in \B{R}^d_{\geq 0}$,
\begin{equation}\label{eq:pJ-l1}
\|\x + \y\|_1 = \|\x\|_1 + \|\y\|_1
\qquad\text{and}\qquad
\|\BF{1}_S\|_1 = |S|.
\end{equation}

\emph{(a) $P_{\mathcal I}$ is closed, convex and downward-closed.}
As the convex hull of finitely many points, $P_{\mathcal I}$ is compact and convex. For downward-closedness, we first show that one coordinate can be decreased. Let $\x \in P_{\mathcal I}$, write $\x = \sum_{A \in \mathcal I} \lambda_A \BF{1}_A$ with $\lambda_A \geq 0$ and $\sum_A \lambda_A = 1$, fix a coordinate $i$, and fix $t \in [0,1]$. Note that
\begin{equation}\label{eq:pJ-xi}
x_i = \sum_{A \in \mathcal I : \, i \in A} \lambda_A .
\end{equation}
For every $A \in \mathcal I$, the set $A \setminus \{i\}$ is independent by the downward-closed property of matroids, so $(1-t)\BF{1}_A + t\BF{1}_{A \setminus \{i\}} \in P_{\mathcal I}$. Taking the convex combination of these points with weights $\lambda_A$, we obtain a point of $P_{\mathcal I}$, namely
\[
\sum_{A} \lambda_A \big( (1-t)\BF{1}_A + t \BF{1}_{A \setminus \{i\}} \big)
= \x - t \sum_{A \ni i} \lambda_A \BF{e}_i
= \x - t x_i \BF{e}_i,
\]
where we used $\BF{1}_A - \BF{1}_{A \setminus \{i\}} = \BF{e}_i$ if $i \in A$ and $= \BF{0}$ otherwise, and then~\eqref{eq:pJ-xi}. As $t$ ranges over $[0,1]$, the $i$-th coordinate $(1-t)x_i$ takes every value in $[0, x_i]$, while the other coordinates are unchanged.
Now let $\x \in P_{\mathcal I}$ and $\BF{0} \leq \y \leq \x$. Starting from $\x$, apply this operation to the coordinates $i = 1, \dots, d$ in turn, lowering the $i$-th coordinate from $x_i$ to $y_i$ (with $t = 1 - y_i/x_i$ if $x_i > 0$, and no change if $x_i = 0$, in which case $y_i = 0$). Each step stays in $P_{\mathcal I}$, and the final point is $\y$. Hence $\y \in P_{\mathcal I}$.

\emph{(b) The $\ell_1$ norm on $P_{\mathcal I}$.}
Every independent set $A$ is contained in a maximal independent set, i.e., a basis, since $V$ is finite; all bases have $\rho$ elements. Hence
\begin{equation}\label{eq:pJ-card}
|A| \leq \rho \text{ for all } A \in \mathcal I, \text{ with equality if and only if } A \text{ is a basis}.
\end{equation}
Let $\x = \sum_A \lambda_A \BF{1}_A \in P_{\mathcal I}$ be any representation as a convex combination. By~\eqref{eq:pJ-l1} and~\eqref{eq:pJ-card},
\begin{equation}\label{eq:pJ-norm}
\|\x\|_1 = \sum_A \lambda_A |A| \leq \rho .
\end{equation}
If $\x \in P_{\mathcal B}$, we may choose the representation using bases only, and then~\eqref{eq:pJ-norm} holds with equality. Conversely, if $\|\x\|_1 = \rho$, then every $A$ with $\lambda_A > 0$ has $|A| = \rho$, so it is a basis by~\eqref{eq:pJ-card}, and $\x \in P_{\mathcal B}$. Since bases are independent, $P_{\mathcal B} \subseteq P_{\mathcal I}$, and therefore
\begin{equation}\label{eq:pJ-PB}
P_{\mathcal B} = P_{\mathcal I} \cap \{\x \in \B{R}^d : \|\x\|_1 = \rho\}.
\end{equation}

\emph{(c) Every point of $P_{\mathcal I}$ is dominated by a point of $P_{\mathcal B}$.}
Let $\x = \sum_A \lambda_A \BF{1}_A \in P_{\mathcal I}$. For each $A$ with $\lambda_A > 0$, choose a basis $B_A \supseteq A$, and let $\y = \sum_A \lambda_A \BF{1}_{B_A}$. Then $\y \in P_{\mathcal B}$, and $\y \geq \x$ because $\BF{1}_{B_A} \geq \BF{1}_A$ for each $A$.

\emph{(d) $(P_{\mathcal I})^{\op m} = P_{\mathcal B} = (P_{\mathcal I})^*$.}
Let $\x \in (P_{\mathcal I})^{\op m}$. By (c), there is $\y \in P_{\mathcal B} \subseteq P_{\mathcal I}$ with $\y \geq \x$, and maximality of $\x$ forces $\y = \x$, so $\x \in P_{\mathcal B}$. Conversely, let $\x \in P_{\mathcal B}$ and suppose that $\y \in P_{\mathcal I}$ satisfies $\y \geq \x$ and $\y \neq \x$. Then $\y - \x \geq \BF{0}$ is nonzero, and by~\eqref{eq:pJ-l1}, $\|\y\|_1 = \|\x\|_1 + \|\y - \x\|_1 > \rho$, which contradicts~\eqref{eq:pJ-norm}. Hence $\x \in (P_{\mathcal I})^{\op m}$, and $(P_{\mathcal I})^{\op m} = P_{\mathcal B}$. Since $P_{\mathcal B}$ is closed and convex, $(P_{\mathcal I})^* = \overline{\op{conv}}(P_{\mathcal B}) = P_{\mathcal B}$. The same argument, with $P_{\mathcal B}$ in place of $P_{\mathcal I}$, shows that no point of $P_{\mathcal B}$ is strictly dominated by another point of $P_{\mathcal B}$, so $(P_{\mathcal B})^{\op m} = (P_{\mathcal B})^* = P_{\mathcal B}$. Together with~\eqref{eq:pJ-PB}, this proves the first line of the theorem.

\emph{(e) $r_1(P_{\mathcal I}) = r_1(P_{\mathcal B}) = \rho$.}
By (d) and~\eqref{eq:pJ-PB}, $r_1(P_{\mathcal I}) = \min_{\vv \in P_{\mathcal B}} \|\vv\|_1 = \rho$ and $r_1(P_{\mathcal B}) = \min_{\vv \in P_{\mathcal B}} \|\vv\|_1 = \rho$.

\emph{(f) $R_1(P_{\mathcal I}) = R_1(P_{\mathcal B}) \leq 2\rho$.}
Since $P_{\mathcal B} \subseteq P_{\mathcal I}$, we have $R_1(P_{\mathcal B}) \leq R_1(P_{\mathcal I})$. Conversely, let $\uu, \vv \in P_{\mathcal I}$ and, by (c), choose $\uu', \vv' \in P_{\mathcal B}$ with $\uu' \geq \uu$ and $\vv' \geq \vv$. Then $\uu' \vee \vv' \geq \uu \vee \vv \geq \BF{0}$, so $\|\uu \vee \vv\|_1 \leq \|\uu' \vee \vv'\|_1 \leq R_1(P_{\mathcal B})$. Taking the maximum over $\uu, \vv$ gives $R_1(P_{\mathcal I}) \leq R_1(P_{\mathcal B})$. Finally, for $\uu, \vv \in P_{\mathcal I}$, each coordinate satisfies $\max\{u_i, v_i\} \leq u_i + v_i$ because $u_i, v_i \geq 0$, so by~\eqref{eq:pJ-l1} and~\eqref{eq:pJ-norm},
\[
\|\uu \vee \vv\|_1 \leq \|\uu + \vv\|_1 = \|\uu\|_1 + \|\vv\|_1 \leq 2\rho .
\]
Taking the maximum over $\uu, \vv \in P_{\mathcal I}$ gives $R_1(P_{\mathcal I}) \leq 2\rho$.
\end{proof}

\section{\texorpdfstring{Comparison with Prior Approximation Ratios for $\sigma$-OSS Maximization}{Comparison with Prior Approximation Ratios for sigma-OSS Maximization}}\label{apx:alpha-comparison}

\begin{figure}[htbp]
\centering
\begin{minipage}{0.32\textwidth}
    \centering
    \includegraphics[width=\linewidth]{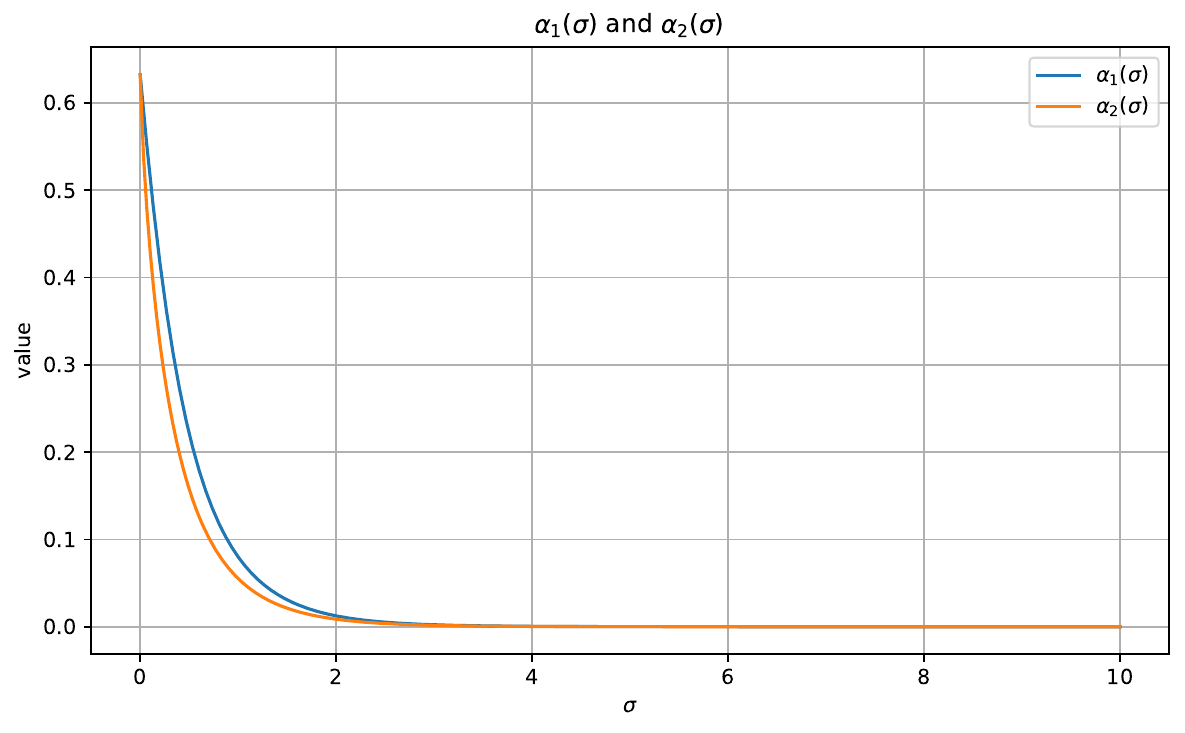}
\end{minipage}
\hfill
\begin{minipage}{0.32\textwidth}
    \centering
    \includegraphics[width=\linewidth]{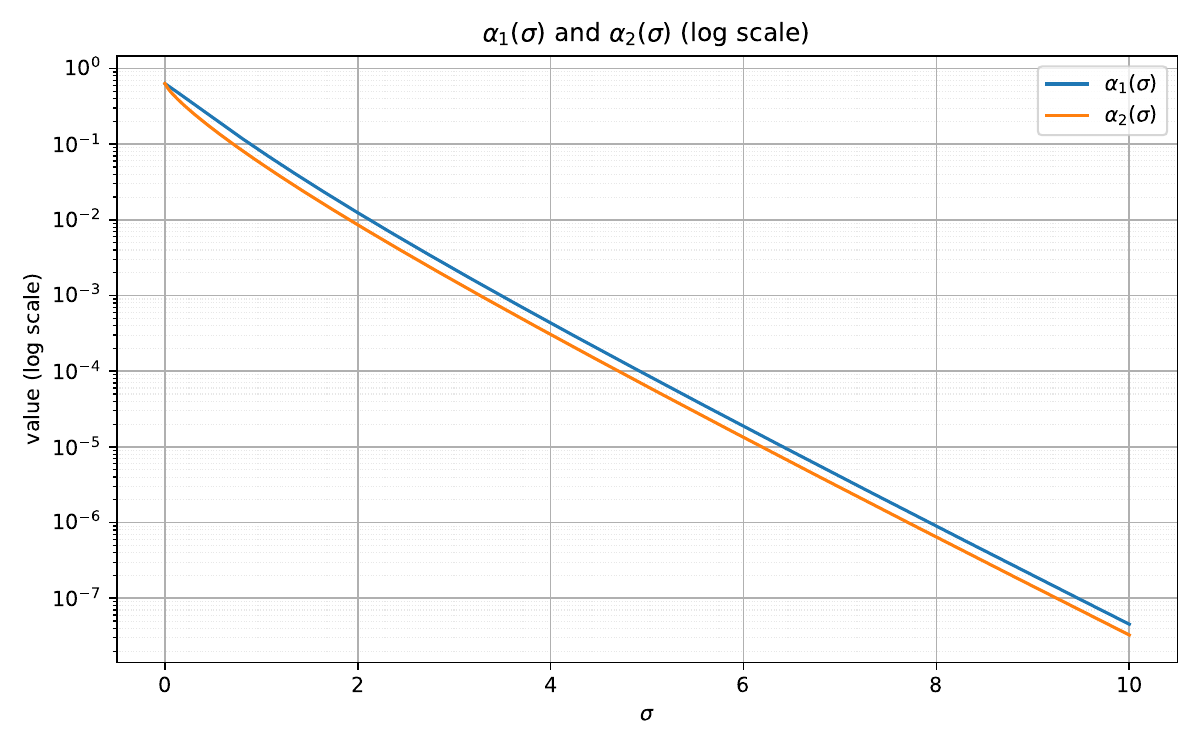}
\end{minipage}
\hfill
\begin{minipage}{0.32\textwidth}
    \centering
    \includegraphics[width=\linewidth]{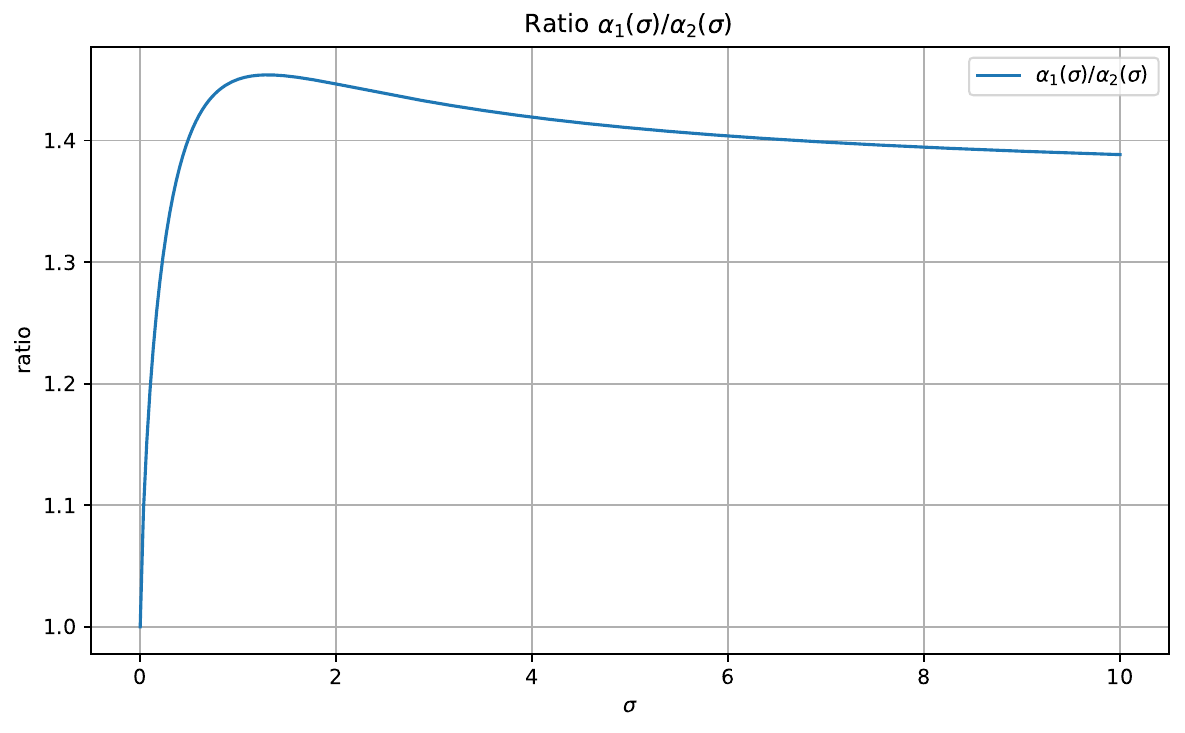}
\end{minipage}
\caption{Comparison of $\alpha_1(\sigma)$ and $\alpha_2(\sigma)$, the same values on a logarithmic scale, and their ratio $\alpha_1(\sigma)/\alpha_2(\sigma)$.}
\end{figure}

\begin{theorem}
Let 
\[
\alpha_1(\sigma) = 1-\exp\!\Big(-\frac{1}{2^{2\sigma}(2\sigma+1)}\Big)
\]
and 
\[
\alpha_2(\sigma) = \sup_{t \in [0, 1)} \left( 1-\exp\left(-(1-t)\left(\frac{t}{t+1}\right)^{2 \sigma} \right)\right),
\]
for $\sigma \geq 0$.
Then we have $\alpha_1(\sigma) \geq \alpha_2(\sigma)$ with equality only at $\sigma = 0$.
Moreover, we have 
\[
\lim_{\sigma \to \infty} \frac{\alpha_1(\sigma)}{\alpha_2(\sigma)} = \frac{e}{2}.
\]
\end{theorem}

\begin{proof}
For $\sigma \geq 0$, let
\[
A(\sigma)=\frac{1}{2^{2\sigma}(1+2\sigma)}
,\qquad 
B(\sigma)=\max_{0\le t<1}(1-t)\Bigl(\frac{t}{1+t}\Bigr)^{2\sigma}.
\]
Note that $\alpha_1(\sigma) = 1 - \exp(-A(\sigma))$ and $\alpha_2(\sigma) = 1 - \exp(-B(\sigma))$.
We will prove that
\[
A(\sigma)\ge B(\sigma)\quad\text{for all }\sigma\ge 0,
\]
with equality if and only if \(\sigma=0\), and that
\[
\lim_{\sigma \to \infty}\frac{A(\sigma)}{B(\sigma)} = \frac{e}{2}.
\]
Given these two facts, $A(\sigma) \to 0$ and $0 \leq B(\sigma) \leq A(\sigma)$ imply that $B(\sigma) \to 0$ and therefore
\begin{align*}
\lim_{\sigma \to \infty} \frac{\alpha_1(\sigma)}{\alpha_2(\sigma)} 
&= \lim_{\sigma \to \infty} \frac{1 - \exp(-A(\sigma))}{1 - \exp(-B(\sigma))} \\
&= \lim_{\sigma \to \infty} \frac{1 - \exp(-A(\sigma))}{A(\sigma)} \frac{B(\sigma)}{1 - \exp(-B(\sigma))} \frac{A(\sigma)}{B(\sigma)} \\
&= \lim_{\sigma \to \infty} \frac{A(\sigma)}{B(\sigma)}
= \frac{e}{2},
\end{align*}
which completes the proof.

\noindent\textbf{Proof of $A(\sigma) \geq B(\sigma)$: }
For \(\sigma=0\), we have
\[
A(0)=1,
\qquad
B(0)=\sup_{0\le t<1}(1-t)=1,
\]
so equality holds.
Now assume \(\sigma>0\). Set
\[
f_\sigma(t):=(1-t)\Bigl(\frac{t}{1+t}\Bigr)^{2\sigma},
\qquad 0\le t<1.
\]
Since \(f_\sigma(0)=0\) and \(f_\sigma(t)\to 0\) as \(t\to 1^{-}\), the maximum is attained at some interior point \(t_\sigma\in(0,1)\). Differentiating,
\[
\frac{d}{dt}\log f_\sigma(t)
=-\frac{1}{1-t}+2\sigma\Bigl(\frac1t-\frac1{1+t}\Bigr)
=-\frac{1}{1-t}+\frac{2\sigma}{t(1+t)}.
\]
Thus \(t_\sigma\) satisfies
\[
-\frac{1}{1-t_\sigma}+\frac{2\sigma}{t_\sigma(1+t_\sigma)}=0,
\]
or equivalently,
\begin{equation}\label{eq:crit}
2\sigma(1-t_\sigma)=t_\sigma(1+t_\sigma).
\end{equation}
Using \eqref{eq:crit},
\[
1+2\sigma
=1+\frac{t_\sigma(1+t_\sigma)}{1-t_\sigma}
=\frac{1+t_\sigma^2}{1-t_\sigma}.
\]
Therefore
\[
\frac{B(\sigma)}{A(\sigma)}
=
2^{2\sigma}(1+2\sigma)(1-t_\sigma)\Bigl(\frac{t_\sigma}{1+t_\sigma}\Bigr)^{2\sigma}
=
(1+t_\sigma^2)\Bigl(\frac{2t_\sigma}{1+t_\sigma}\Bigr)^{2\sigma}.
\]
Again by \eqref{eq:crit},
\[
\log\frac{B(\sigma)}{A(\sigma)}
= 
\log(1+t_\sigma^2) + \frac{t_\sigma(1+t_\sigma)}{1-t_\sigma} \log\Bigl(\frac{2t_\sigma}{1+t_\sigma}\Bigr)
=
h(t_\sigma),
\]
where
\[
h(t):=\log(1+t^2)+\frac{t(1+t)}{1-t}\log\!\Bigl(\frac{2t}{1+t}\Bigr),
\qquad 0<t<1.
\]
We claim that \(h(t)<0\) for all \(t\in(0,1)\). 
Set
\[
u=\frac{1-t}{1+t}\in(0,1).
\]
Then
\[
t=\frac{1-u}{1+u},
\qquad
\frac{2t}{1+t}=1-u,
\qquad
1+t^2=\frac{2(1+u^2)}{(1+u)^2},
\qquad
\frac{t(1+t)}{1-t}=\frac{1-u}{u(1+u)}.
\]
Thus
\[
h(t)=\log\!\Bigl(\frac{2(1+u^2)}{(1+u)^2}\Bigr)
+\frac{1-u}{u(1+u)}\log(1-u).
\]
Using \(\log(1-u)\le -u\) for \(0 < u < 1\), we get
\[
h(t)\le \log\!\Bigl(\frac{2(1+u^2)}{(1+u)^2}\Bigr)-\frac{1-u}{1+u}.
\]
Define
\[
\phi(u):=\log\!\Bigl(\frac{2(1+u^2)}{(1+u)^2}\Bigr)-\frac{1-u}{1+u},
\qquad 0 < u \leq 1.
\]
Differentiating,
\begin{align*}
\phi'(u)
&=
\frac{d}{du}\Bigl[\log(2(1+u^2)) - 2\log(1+u)\Bigr]
-\frac{d}{du}\Bigl(\frac{1-u}{1+u}\Bigr). \\
&= \frac{2u}{1+u^2} - \frac{2}{1+u}
+ \frac{2}{(1+u)^2} \\
&= \frac{2u(1+u)^2 - 2(1+u)(1+u^2) + 2(1+u^2)}{(1+u)^2(1+u^2)} \\
&= \frac{2u(1+2u+u^2) - 2(1+u+u^2+u^3) + 2(1+u^2)}{(1+u)^2(1+u^2)} \\
&= \frac{4u^2}{(1+u)^2(1+u^2)} > 0
\end{align*}
Hence \(\phi\) is increasing on \((0,1]\). 
Since $\phi(1)=\log 1-0=0$, it follows that $\phi(u) < 0$, for all $0 < u < 1$.
Therefore \(h(t)<0\) for every \(t\in(0,1)\), and so
\[
\log\frac{B(\sigma)}{A(\sigma)}=h(t_\sigma)<0.
\]
Thus, we have shown that
\[
A(\sigma) > B(\sigma),\quad\text{for all }\sigma > 0.
\]
Together with the case \(\sigma=0\), this proves
\[
A(\sigma)\ge B(\sigma),\quad\text{for all }\sigma\ge 0,
\]
with equality if and only if \(\sigma=0\).

\noindent\textbf{Computing the limit of $A(\sigma)/B(\sigma)$: }
It remains to compute the asymptotic ratio. 
Define
\[
\eta_\sigma:=\frac{1}{1-t_\sigma}.
\]
Then $t_\sigma=1-\frac{1}{\eta_\sigma}$, and so
\[
t_\sigma(1+t_\sigma)
=\left(1-\frac1{\eta_\sigma}\right)\left(2-\frac1{\eta_\sigma}\right)
=2-\frac{3}{\eta_\sigma}+\frac{1}{\eta_\sigma^2}.
\]
Substituting into \eqref{eq:crit} gives
\[
\frac{2\sigma}{\eta_\sigma}
=
2-\frac{3}{\eta_\sigma}+\frac{1}{\eta_\sigma^2}.
\]
By multiplying both sides by $\eta_\sigma/2$ and rearranging, we see
\[
\eta_\sigma = \sigma + \frac{3}{2} - \frac{1}{2 \eta_\sigma}.
\]
Since $t_\sigma \in (0, 1)$, we have $\eta_\sigma > 1$. Therefore
\[ | \eta_\sigma - \sigma | = \frac{3}{2} - \frac{1}{2 \eta_\sigma} < 2. \]
Now
\[
B(\sigma)
= (1-t_\sigma)\Bigl(\frac{t_\sigma}{1+t_\sigma}\Bigr)^{2\sigma}
= \frac{1}{\eta_\sigma}
\left(\frac{1-\eta_\sigma^{-1}}{2-\eta_\sigma^{-1}}\right)^{2\sigma}
= \frac{2^{-2\sigma}}{\eta_\sigma}
\left(1 - \frac{1}{2\eta_\sigma - 1}\right)^{2\sigma}.
\]
Since $| \eta_\sigma - \sigma | < 2$, we see that
\begin{align*}
\lim_{\sigma \to \infty} 2\sigma \log \left(1 - \frac{1}{2\eta_\sigma - 1}\right)
&= \lim_{\sigma \to \infty} \frac{2\sigma}{2\eta_\sigma - 1} (2\eta_\sigma - 1) \log \left(1 - \frac{1}{2\eta_\sigma - 1}\right) \\
&= \lim_{\sigma \to \infty} \frac{\log \left(1 - \frac{1}{2\eta_\sigma - 1}\right)}{\frac{1}{2\eta_\sigma - 1}} = -1,    
\end{align*}
where the last equality follows from the standard limit $\lim_{x\to 0}\frac{\log(1-x)}{x} = -1$.
Therefore
\[
\lim_{\sigma \to \infty} \frac{B(\sigma)}{2^{-2\sigma} / \sigma}
= \lim_{\sigma \to \infty} \frac{\sigma}{\eta_\sigma}
\left(1 - \frac{1}{2\eta_\sigma - 1}\right)^{2\sigma}
= \lim_{\sigma \to \infty} \left(1 - \frac{1}{2\eta_\sigma - 1}\right)^{2\sigma}
= e^{-1}.
\]
Thus
\[
B(\sigma)\sim \frac{2^{-2\sigma}}{e\,\sigma}.
\]
On the other hand,
\[
A(\sigma)=\frac{1}{2^{2\sigma}(1+2\sigma)},
\]
so
\[
\frac{A(\sigma)}{2^{-2\sigma}/\sigma}
=\frac{\sigma}{1+2\sigma}\longrightarrow \frac12.
\]
Therefore
\[
\frac{A(\sigma)}{B(\sigma)}
=
\frac{A(\sigma)}{2^{-2\sigma}/\sigma}\cdot
\frac{2^{-2\sigma}/\sigma}{B(\sigma)}
\longrightarrow
\frac12\cdot e
=\frac e2.
\]
Hence
\[
\lim_{\sigma\to\infty}\frac{A(\sigma)}{B(\sigma)}=\frac e2.
\]
\end{proof}

\end{document}